\documentclass{article}

\PassOptionsToPackage{numbers,sort&compress}{natbib}
\usepackage{iclr2027_conference,times}
\iclrfinalcopy 
\usepackage[utf8]{inputenc}
\usepackage[T1]{fontenc}
\usepackage{amsmath}
\usepackage{amssymb}
\usepackage{amsthm}
\usepackage{booktabs}
\usepackage{graphicx}
\usepackage[hidelinks]{hyperref}
\usepackage{microtype}
\usepackage{nicefrac}
\usepackage{xcolor}
\usepackage{url}
\usepackage{comment}

\usepackage{tcolorbox}
\tcbuselibrary{theorems,skins,breakable}
\usepackage{pgfplots}
\pgfplotsset{compat=1.17}
\usepackage{subcaption}

\newcommand{\flip}{\mathrm{Flip}}
\newcommand{\context}{\mathcal{C}}

\usepackage{bbold}
\usepackage{titletoc}
\title{Same Text, Different Prediction: Serving-Context Nondeterminism in Text Classifiers\thanks{Accepted at the LIGHT, ODI, and TAE workshops at NeurIPS 2026.}}
\author{%
\begin{tabular*}{\textwidth}{@{\extracolsep{\fill}}lcr@{}}
Santhosh Kumar Kasa & Siva Rajesh Kasa & Sumit Negi\\[2pt]
\multicolumn{3}{c}{\normalfont Amazon}
\end{tabular*}}
\date{}

\begin{document}
\maketitle
\fancyhead{} 
\renewcommand{\headrulewidth}{0pt}

\begin{abstract}

Deterministic inference is essential for reliable and trustworthy machine learning. Prior studies of text generation have shown that changing factors such as batch size, batch composition, hardware, or inference engine can alter the generated text, even when the prompt, model parameters, and sampling randomness are fixed. These differences have been attributed in part to floating-point non-associativity, shape-dependent kernel selection, and other implementation-level differences in numerical execution. However, it remains unclear whether, when, and to what extent the same factors affect text classification. We present a systematic study of serving-context non-invariance in text classifiers, which prior work has measured only through generated text. We train 180 models spanning discriminative, pseudo-generative, and fully generative classifier formulations and evaluate each across four categories of serving contexts, holding the checkpoint and the text fixed. Label stability does not imply score stability. Changing only the batch shape changes no labels across fp32 comparisons, yet under bf16 it moves up to 56.7 percentage points of predicted probability mass, with label changes concentrated at small margins. Fully generative classifiers change more labels than their discriminative counterparts under the same serving changes. We derive sufficient conditions for label stability under each serving change and give a separate mitigation for each mechanism. Our results identify and quantify the serving conditions that must be fixed for reproducible text classification.

\end{abstract}

\section{Introduction}

A large language model given the same prompt, the same parameters, and the
same sampling seed is widely expected to produce the same output, yet in
practice it often does not. With every user-visible input fixed, the
generated text can still change when the batch size or batch composition
changes, when numerical precision changes, when the hardware changes, or when
one inference engine is exchanged for another
\citep{yuan2025understanding,pape2026silent}. This phenomenon has been termed \textit{nondeterminism}. One documented cause of nondeterminism is
numerical. Floating-point addition is not associative
\citep{goldberg1991floating}, and serving systems select kernels and
reduction orders as a function of batch shape and load, so the same logical
computation is executed as a different sequence of rounded operations
\citep{yuan2025understanding,he2025defeating,zhao2026corun}. A second cause
is the inference backend itself. Implementation choices such as paged
attention, custom CUDA kernels, and optimized memory management introduce
divergence large enough that switching backends alone, with model weights,
prompts, decoding parameters, and hardware held constant, shifts benchmark
scores by several percentage points, and by up to 16.6 in the extreme case
\citep{pape2026silent}. In autoregressive generation these perturbations
compound, because a single altered token early in the sequence conditions
every token that follows. Nearly all of this evidence, however, concerns text
generation. Whether the same serving-context changes affect the single
forward pass of a text classifier has received almost no direct attention or measurement.

Text classification is among the most widely deployed applications of neural
language models, spanning sentiment analysis
\citep{socher2013recursive,maas2011learning}, topic classification
\citep{zhang2015character}, emotion detection \citep{saravia2018carer}, etc. Classifiers operate
in both online real-time and offline batched settings, embedded in
production pipelines for relevance estimation and ranking
\citep{nogueira2019passage}, filtering \citep{sahami1998bayesian}, and
content moderation \citep{wulczyn2017exmachina,lees2022perspective}, where
they form part of larger machine learning systems whose behavior must be
maintained over time \citep{sculley2015debt}. Transformer-era classifiers
also come in distinct paradigms. Following the terminology of
\citet{kasa2025generative}, discriminative classifiers attach a dedicated
classification head to a pretrained encoder \citep{devlin2019bert};
pseudo-generative classifiers are trained with a language-model objective and
predict a class token in a single forward pass; and fully generative
classifiers compare the likelihood of the input under each candidate class, a
distinction with a long history from classical statistics to neural models
\citep{efron1975efficiency,ng2001discriminative,yogatama2017generative,zheng2023revisiting}.
These formulations differ measurably in propensity to shortcut learning
\citep{li2025shortcut,kasa2026shortcut}, sample efficiency
\citep{ng2001discriminative,zheng2023revisiting}, calibration
\citep{guo2017calibration}, and privacy risk
\citep{makroo2026hidden}, so a serving-time perturbation need not affect them
equally.

In every one of these deployments, serving-context nondeterminism can change the final probabilities without any visible change to the model or the text, with direct consequences such as a reclassified review, a rerouted document, or a filtered message. Yet deployed classifiers are commonly monitored through
aggregate metrics such as accuracy or F1, which are invariant to which examples
flip. 
This motivates our present study. We ask whether the serving-context effects documented for generation transfer to classification, and whether the apparent invariance of classification reflects stability or only the absence of measurement. 
We present a systematic study
of serving-context sensitivity in neural text classifiers, conducted through
large-scale controlled experimentation. 
The contributions of this work are as follows:
\begin{enumerate}

\item \textbf{First systematic empirical study.} We organize serving-context factors into a taxonomy and measure their effects. We train 180 text classifiers from scratch (four transformer classification formulations, three sizes, five datasets, three seeds) and evaluate each across 3,960 serving contexts (padded length, batch size, batch composition, precision, quantization, device, backend), holding the checkpoint and text fixed so any change is attributable to the setting rather than run-to-run randomness. Two findings follow. Changing the padded length alone changes about 10\% of predicted labels, more than any other factor, and settings that change almost no labels still shift class probabilities by tens of points, so label accuracy does not reveal the instability.


\item \textbf{Mitigations.} We show that each mechanism has a targeted fix
and measure its invariance--cost trade-off. 

\item \textbf{Theory.} We formalize the conditions under which logit and probability drift can lead to label flipping. We also give sufficiency conditions to rule out padding-based non-determinism. Based on the empirical results, we formalize the
three-way tension among inference speed, task accuracy, and serving-context
invariance into an impossibility condition: the optimizations that make serving fast, including dynamic
batching, reduced precision, quantization, and shape-dependent kernel
selection, are the mechanisms that threaten invariance, while
remedies that restore invariance can cost speed or
accuracy.

\end{enumerate}

\section{Related Work}

\textbf{Discriminative, generative, and pseudo-generative classifiers.}
The comparison between discriminative and generative classifiers has origins in the
classical analyses of logistic regression, discriminant analysis, naive Bayes,
and related estimators \citep{efron1975efficiency,ng2001discriminative,
xue2008comment}. \citet{yogatama2017generative} extended the comparison to neural text
classification, \citet{rezaee2021discriminative} studied transformer
implementations, and \citet{zheng2023revisiting} developed a broader
theoretical analysis.
Recent work studies language models used directly for text classification
\citep{sun2023text} and provides a transformer-era comparison of encoder,
autoregressive, masked-language-model, and discrete-diffusion classifiers
\citep{kasa2025generative}. The masked-language-model formulation is related
to pseudo-likelihood estimation \citep{wang2019bert}. \citet{kasa2026shortcut} show
that fully generative classifiers are not inherently immune to shortcut
learning, while \citet{makroo2026hidden} show that modeling the input distribution
can increase membership-inference risk. These studies compare training
objectives, model formulations, sample efficiency, robustness, calibration,
and privacy. They do not examine whether a fixed trained classifier output changes
when padding, batching, numerical precision, quantization, device, or backend
changes at inference time.

\textbf{Non-determinism in language-model inference.}
For text generation, prior work documents that outputs change when batch
size, precision, hardware, or the inference engine changes, even with the
prompt, parameters, and sampling seed fixed
\citep{yuan2025understanding,pape2026silent}. Repeated calls to hosted
models under nominally deterministic settings show the same variability in
the generated answer strings, even when those strings are then parsed into
task labels \citep{atil2024non}. The numerical cause is that
floating-point addition is not associative \citep{goldberg1991floating} and
serving systems select kernels and reduction orders as a function of batch
shape \citep{he2025defeating,zhao2026corun}. Backend implementation choices
add further divergence \citep{pape2026silent}. Dataset-level metrics
attenuate exactly the per-example movement these effects produce
\citep{fang2026dataset}. Remedies proposed for generation include
batch-invariant kernels \citep{he2025defeating}, fixed-shape scheduling
\citep{zhao2026corun}, and verified recomputation of the decoded steps
that disagree \citep{gond2026llm}. This literature studies generation, where
autoregressive feedback amplifies a one-token change. Whether and when the
same mechanisms appear in the single forward pass of a classifier is the
question we study.

\textbf{Extended discussion.}
Appendix~\ref{app:related} expands this section: an extended account of the three
classifier families and of non-determinism in language-model inference, and a
discussion that situates the present study against prediction instability and
churn, imprecise probability and calibration, the numerical and positional
mechanisms the paper relies on, and behavioral and production-readiness testing.

\section{Serving Context and Predictive Claims}
\label{sec:method}

Let $p_\theta(\cdot\mid x,c)$ denote the class-probability vector for target
$x$, checkpoint $\theta$, and serving context $c$. The example whose prediction we track is the target. A context records every implementation
setting that can vary while the trained parameters and the text stay fixed:
the software backend, the device, the numerical precision, the quantization
configuration, the batch size, the order in which examples are processed,
the position an example occupies within its batch, the other examples that
share the batch, the padding policy, and the rule that assigns position
indices to tokens. For a declared finite set of contexts $\context$, our comparisons observe
selected pairs drawn from $\context$. Each pair fixes the checkpoint and the
text and changes one recorded setting. These pairs are controlled witnesses
of how far the prediction can move, not a guarantee over every serving
configuration, and we place no distribution over the contexts.

Let $f_\theta(x;c)=\arg\max_y p_\theta(y\mid x,c)$. For reference context
$c_0$, the paired flip indicator is $\flip_\theta(x;c_0,c)
  =\mathbb{1}\!\left[f_\theta(x;c_0)\neq f_\theta(x;c)\right].$
The flip rate averages this indicator over the targets that both contexts scored.
 It measures the decision changes a user would observe rather than aggregate quality.

Labels can remain stable while scores move. For aligned probability vectors
$p_0$ and $p_c$, we measure total variation $\mathrm{TV}(p_0,p_c)=\frac{1}{2}\sum_k |p_{0,k}-p_{c,k}|.$
We also report the maximum absolute logit change
$\lVert z_0-z_c\rVert_\infty$. Score drift is computed both over all aligned
targets and over the subset whose top label does not change.

We separate two paths by which context can reach the predictive distribution:
$ p_\theta(\cdot\mid x,c)
 =N\!\left(E_c\!\left(\theta,P_c(x)\right)\right).$
$P_c$ constructs token IDs, masks, position IDs, padding, and tensor layout;
$E_c$ executes the checkpoint and returns class scores; and $N$ applies the
fixed class normalization. A numerical intervention changes $E_c$ while
preserving the model-relevant output of $P_c$. A structural intervention changes $P_c(x)$ itself, as when a real token's position index is read from its tensor column and shifts with the padded width. This decomposition supplies one of the two axes we use to organize every
serving factor. The first axis is the \emph{mechanism} by which a factor can
move the prediction, and it is the split above. A factor is
structural if it changes $P_c$, the represented input, and numerical if it
changes only $E_c$, the execution of a fixed represented input. The second
axis is the \emph{locus} at which a factor is set. A deployment-level factor
is chosen once when the model is put into service and then held fixed across
requests, whereas a request-level factor varies from one request or batch to
the next. Every factor we study occupies one cell of the resulting table
(Table~\ref{tab:taxonomy}). Padding is the single factor that spans a
boundary. Its policy is deployment-level, but the width it produces under a
dynamic policy is realized per request, and it is the only request-realized
factor whose mechanism is structural, so we report it on its own. This table
organizes both the design below and the results that follow.

Each comparison changes exactly one recorded setting and holds the checkpoint,
the dataset split, the evaluated examples, and every other recorded setting
fixed. A comparison whose reference run is missing is discarded rather than
substituted from another batch size or seed, and when the intervention is which
examples share the batch, the padded width is held fixed so that changing the
other examples cannot change the tensor shape.

\section{Experimental Design}
\label{sec:design}

The experiments address four questions: (a) does a prediction measured in one serving context still hold in the others;
(b) do the predicted label, the class scores, and the decision margin agree
about stability; (c) which of the factors move the prediction by executing the same input differently and which by changing the input itself; and (d) what evidence
distinguishes a working mitigation from prediction collapse or from a
comparison that applies no effective intervention?

\textbf{Models and data.} We train 180 classifiers from random
initialization: four classifier formulations built on BERT or GPT-2, three
jointly scaled sizes, five datasets (SST-2, SST-5, AG News, Emotion, and IMDb
\citep{socher2013recursive,zhang2015character,saravia2018carer,maas2011learning}),
and three random seeds, reusing public tokenizers but no pretrained weights.
The small, medium, and large sizes scale the feed-forward width, depth, and
attention heads jointly, following the three jointly scaled sizes of
\citet{kasa2025generative} and \citet{makroo2026hidden}. The exact
configurations, objectives, and optimization are in
Appendix~\ref{app:implementation}.

\textbf{Serving interventions.} The taxonomy of
Table~\ref{tab:taxonomy} sorts the comparisons into three groups, each
departing from the reference deployment of full precision, GPU execution, and
static padding. Deployment-level comparisons hold the request context at a
single example and change one deployment factor: the
numerical precision (full precision versus the reduced-precision format bf16,
with fp16 reported in the appendix) and the 8-bit integer scale computation
(an offline fixed scale versus a scale recomputed from each batch).
Request-level comparisons instead hold the deployment fixed and change one
request factor: the batch size, the padded width (static versus dynamic), the
position the target occupies within its batch, the order in which examples are
processed, and the other examples that share the batch at a fixed width. The
order comparison permutes the full test list before fixed-width batching, so
it can change the target's peers and slot but not the tensor width; the
composition comparison changes the peers at a fixed width; a separate
random-peer experiment lets peer lengths change the dynamic width and is a
joint peer and shape intervention. A
third comparison changes device and 8-bit integer quantization together, from
GPU full precision to CPU integer execution; because it moves several factors
at once it is reported as a whole-deployment migration rather than as an
estimate of any single factor, and it is the change a practitioner makes in a
single step. Padding comparisons keep the target and the other examples in its
batch fixed and change dynamic versus static width under each family's existing
padding side. They do not randomize left versus right padding within a
checkpoint.

\textbf{Model-size analysis.} The model-size analysis compares small, medium, and large models and retains
a family, dataset, and seed combination only when all three sizes exceed
1.15 times chance validation accuracy, leaving 40 matched combinations
overall and 16 for the GPT-2 padding comparison. We average repeated
comparisons within each checkpoint and size, weight each retained
combination equally, and resample combinations while preserving the
three-size pairing. 


\definecolor{sinavy}{RGB}{46,64,101}
\definecolor{sipine}{RGB}{56,92,74}
\definecolor{siwine}{RGB}{114,47,55}
\newtcbtheorem[number within=section]{cthm}{Theorem}%
{enhanced,breakable,colback=sinavy!4!white,colframe=sinavy,boxrule=0.4pt,
 arc=1.5pt,left=5pt,right=5pt,top=2pt,bottom=2pt,beforeafter skip=6pt,
 fonttitle=\bfseries,coltitle=sinavy,
 attach title to upper,after title={.\ }}{thmx}
\newtcbtheorem[number within=section]{cprop}{Proposition}%
{enhanced,breakable,colback=sipine!4!white,colframe=sipine,boxrule=0.4pt,
 arc=1.5pt,left=5pt,right=5pt,top=2pt,bottom=2pt,beforeafter skip=6pt,
 fonttitle=\bfseries,coltitle=sipine,
 attach title to upper,after title={.\ }}{propx}

\section{Results}

We report every effect against a matched reference that fixes the checkpoint
and the text and changes one recorded setting. The reference deployment is full
precision on GPU with static padding. The deployment-level and batch-size
comparisons are measured against a single example, and the within-batch
comparisons against the canonical arrangement at a fixed batch. Each comparison
records three quantities on the same target: whether the
predicted class changes, how far the class probabilities move (the total
variation distance, a value between 0 and 1), and how far the unnormalized
class scores move (the largest absolute change among the logits). Because a
classifier that always predicts one class would look perfectly stable while
being uninformative, we keep only the 158 of 180 trained models whose validation
accuracy exceeds random guessing by a 15\% relative margin. Repeating the identical context reproduces the logits bitwise
on every hardware and software configuration we test (Table~\ref{tab:floor}), so each reported
change is caused by the setting we vary, not by run-to-run variation.

We present the comparisons in the order of the taxonomy
(Table~\ref{tab:taxonomy}): first the deployment-level changes a model owner
chooses once, then the request-level changes that vary with traffic after a
deployment is fixed, and finally the combined migration a practitioner makes in
one step. Two findings hold across all three. First, a stable predicted label
routinely coexists with movement in the class probabilities and scores that downstream
thresholds, ranking, and filtering consume, so we report label changes and
score movement together throughout. Second, the two mechanisms of
Table~\ref{tab:taxonomy} behave differently. A change to the represented input
(the structural mechanism) moves labels across a wide range of confidence,
whereas a change to execution alone (the numerical mechanism) moves labels only
where the two leading classes are nearly tied.

\subsection{Deployment-level changes}
\label{sec:deployment}

Moving a fixed model to a different deployment while holding the request at the
reference, a single example at static width, can change the prediction even
though there is no batching, no neighboring example, and no change to the
input. Figure~\ref{fig:v2deploy} reports two such single-factor changes.
Reducing the numerical format from full precision to bf16 leaves the
discriminative BERT classifier's labels unchanged, but flips about one percent
of the fully generative GPT-2 classifier's labels and moves one of its scores
by up to 0.56. Replacing the offline fixed 8-bit integer scale with a scale
recomputed from each batch is the larger change. It flips 4 of 192 fully generative labels and moves a score by up to 0.84, again with no observed
change for BERT. Precision and quantization are distinct changes. Reduced
precision (bf16) is a floating-point format with fewer mantissa bits, whereas
the 8-bit integer path represents activations as integers through a scale. Here
it is the choice of scale rather than the format that produces the larger effect. The
gap between the discriminative and the fully generative formulation appears again
throughout the paper and is largest under the combined migration of
Section~\ref{sec:migration}.

\begin{figure}[t]
\centering
\begin{minipage}[t]{0.46\linewidth}
\vspace{0pt}
\centering
\includegraphics[width=0.80\linewidth]{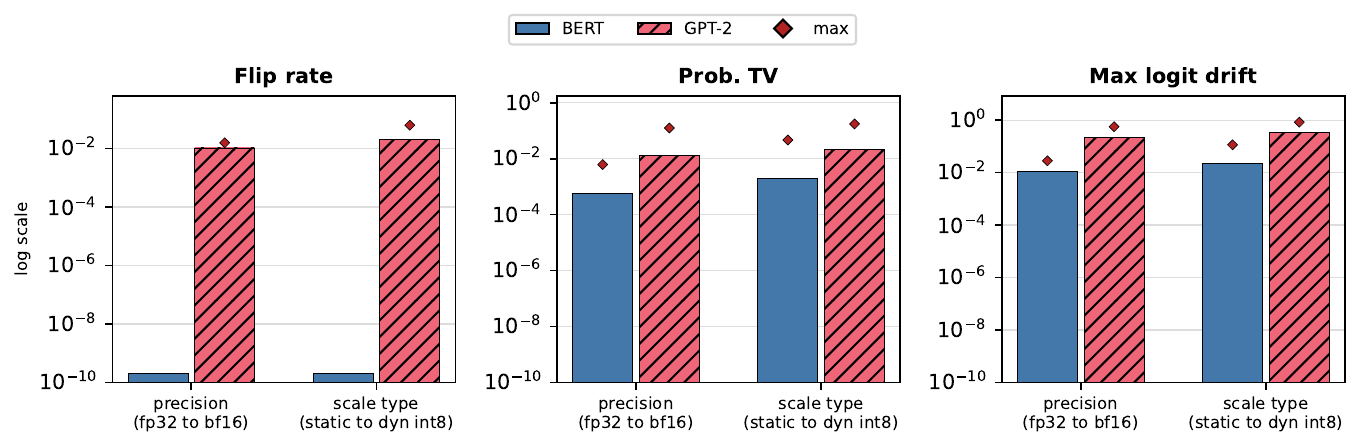}
\caption{Deployment-level changes at the reference request context (one
example, static width): reduced precision (full precision to bf16) and the
8-bit integer scale computation (offline fixed to per-batch).}
\label{fig:v2deploy}
\end{minipage}\hfill
\begin{minipage}[t]{0.5\linewidth}
\vspace{0pt}
\centering
\captionof{table}{Two axes for the serving factors. Rows: deployment-level
(fixed once per service) and request-level (varied per request). Columns:
structural (the represented input changes) and numerical (only execution
changes). Padding spans the boundary.}
\label{tab:taxonomy}
\tiny
\setlength{\tabcolsep}{2pt}
\begin{tabular}{@{}p{0.2\linewidth}p{0.3\linewidth}p{0.42\linewidth}@{}}
\toprule
 & Structural & Numerical \\
\midrule
Deployment-level & padding policy & precision, quantization, device, backend \\
Request-level & padded width & batch size, slot, order, composition \\
\bottomrule
\end{tabular}
\end{minipage}
\end{figure}

\subsection{Request-level serving context within a fixed deployment}
\label{sec:request}

Once a deployment is fixed, the settings that still vary from one request or
batch to the next are the batch size, the padded width, the target's position
within its batch, the order of the other examples, and which other examples
share the batch. Figure~\ref{fig:v1knobs} reports the two that change the tensor shape, batch size and
padded width, across three deployments. The
within-batch rearrangements are treated separately below.

\begin{figure}[t]
  \centering
  \includegraphics[width=0.80\linewidth]{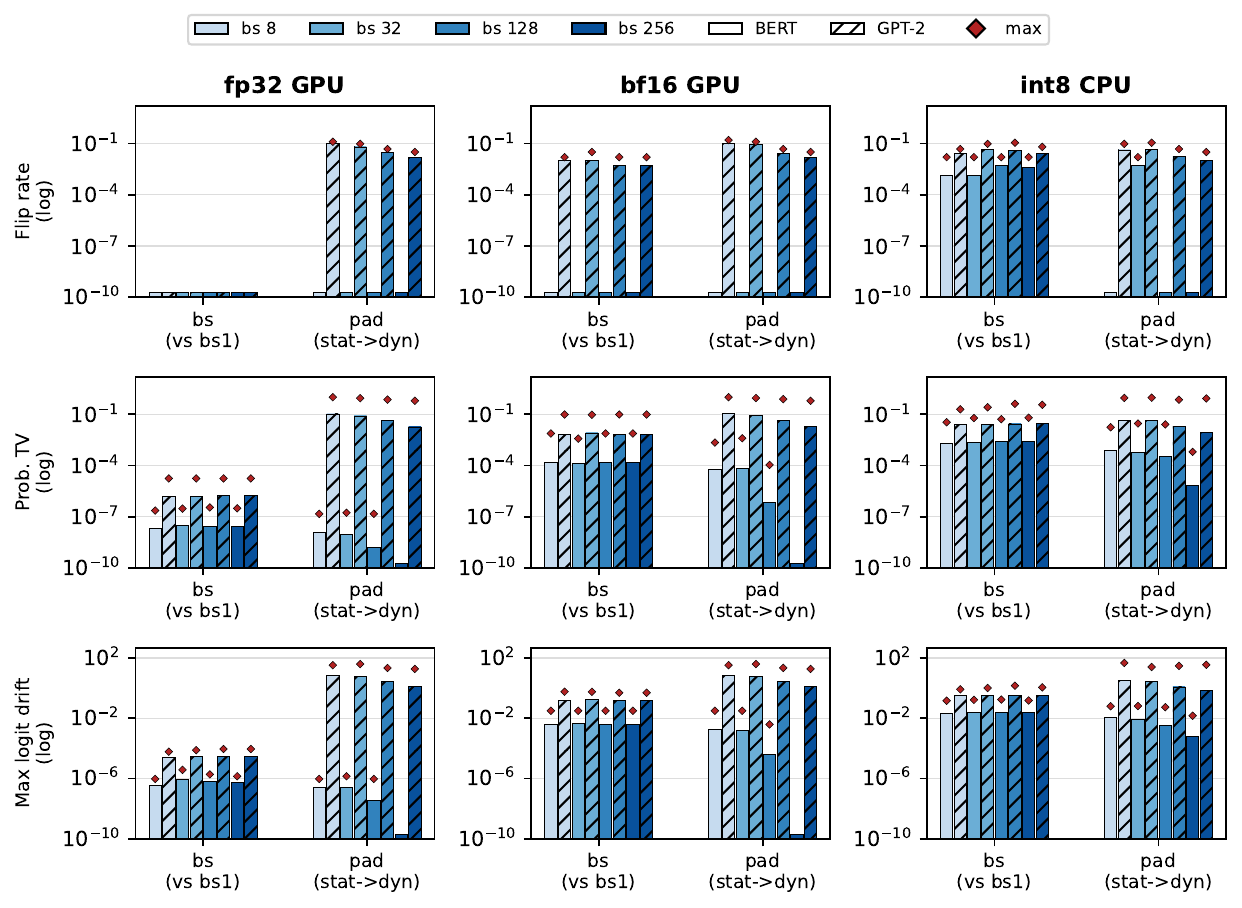}
  \caption{Request-level shape changes within a fixed deployment: batch size
  (against the single-example reference) and padded width (static to dynamic),
  for full-precision GPU, bf16 GPU, and 8-bit integer CPU. Padded width is the
  dominant effect in every deployment; batch size is near the floor at full
  precision and grows only under reduced precision and 8-bit integer
  execution.}
  \label{fig:v1knobs}
\end{figure}

\textbf{Padding is the dominant request-level effect, and it is structural.}
\label{sec:padding-localization}
Changing the padded width is the largest request-level effect on every quantity
in Figure~\ref{fig:v1knobs}, and it is the only one that changes the
represented input rather than its execution. Every padding-induced label change
occurs in the two GPT-2 families, about one in five of their comparisons, while
the right-padded BERT and masked-language-model classifiers change none.
Because the attention mask already blocks padding tokens in every model, the
difference is where the real tokens sit. Right padding leaves their position
indices fixed at every width, while left padding shifts each real token to a
higher tensor column. Two pretrained models separate the padding side from the
position representation. Left-padded Qwen3 with rotary positions
\citep{qwenteam2025qwen3} and right-padded ModernBERT
\citep{warner2024modernbert} both change no labels, so left padding alone is not
the cause (Table~\ref{tab:architecture}). The exact counts, the pretrained-model
replication, and a no-position-encoding boundary case are in
Appendix~\ref{sec:app-arch}.

\textbf{Why width-dependent absolute positions change the input.}
\label{sec:position-intervention}
For a real sequence of length $L$ in a left-padded tensor of width $T$, the
positions read from the tensor column are $(T-L,\ldots,T-1)$. Changing $T$
therefore changes the learned position embedding added to every real token,
so the model receives a different input even though the text is unchanged.
Assigning positions instead by each token's order among the real tokens,
$(0,\ldots,L-1)$, does not depend on the width. Two matched interventions confirm
the cause. Assigning positions by real-token order removes the padding effect,
and deliberately forcing the width-dependent positions back restores it, on
both GPU and CPU. Section~\ref{sec:theory} states the width-invariance result
formally, and Appendix~\ref{sec:app-position} reports the per-seed rates and the
intervention figure (Figure~\ref{fig:intervention}).

\textbf{Batch size matters only through the numerical format.}
The batch-shape effect documented in the generation literature transfers to a single classifier forward pass only under a reduced numerical format. At full precision, changing only the batch shape flips no labels
across 226{,}104 comparisons and moves probabilities by less than $10^{-5}$,
whereas under bf16 it moves one prediction's probability by up to 56.7
percentage points, and the 8-bit integer path is larger still
(Figure~\ref{fig:v1knobs}). Batch size is therefore a rounding effect at full
precision and a measurable effect only once the format is reduced. The margin
stratification and the label-count effect are in Appendix~\ref{sec:app-numerical}.

\textbf{The remaining within-batch rearrangements move no label, with one exception.}
Reordering the other examples is bitwise identical at every precision, because
it never touches the target's own row. Swapping which other examples share the
batch is bitwise identical under floating-point execution for the same reason:
no operation mixes one example's rows into another's. Under dynamic 8-bit
integer quantization, however, it is the one within-batch change that moves a
label. The activation scale is computed over the whole batch, so the content
of the other examples reaches the target through that shared scale.
Figure~\ref{fig:v1comp} shows this composition effect. It grows with batch
size, reaching about 3.7 percent of fully generative labels and a score
movement of 1.3 at a batch of 256. Appendix~\ref{sec:quantization} details the
mechanism and its removal under a fixed scale.

\textbf{The target slot is invariant only up to the kernel's reduction order.}
Moving the target to a different row at a fixed batch shape leaves the output unchanged in exact
arithmetic yet not bitwise on the GPU. A linear layer computes each output
coordinate as a sum over the hidden dimension, $y_n = \sum_{k} a_k W_{k,n}$, so
the target row's value depends only on its own activations and never on the row
it occupies. The GPU matrix multiplication, however, folds the batch and the
sequence into one dimension $M$ and evaluates that sum in an order that depends
on the row's offset in $M$, so changing the slot changes the summation order
and, because floating-point addition is not associative, moves the logits by
about $10^{-5}$ with no label change. The integer CPU path uses one fixed reduction order, so the logits are bitwise identical across slots. Batch-slot invariance is therefore a property of the kernel's
reduction order rather than a mathematical guarantee, which is why the
batch-invariant kernels of Section~\ref{sec:mitigations} remove it.
Appendix~\ref{sec:app-slot} gives the full account.

\textbf{Model size and score drift under unchanged labels.}
Among predictions whose label does not change under padding, the probability
movement roughly doubles as the model grows, so a larger model can appear more
label-stable while the scores under those labels move as much or more. The
label-change rate itself is not monotone in size. Because size, width, depth,
and training quality change together, this is a descriptive association rather
than a scaling law, and the per-size breakdown is in
Figures~\ref{fig:app-knobs-size} and~\ref{fig:app-within-size}.

\subsection{The combined effect of a whole-deployment migration}
\label{sec:migration}

A practitioner rarely changes one factor at a time. A common change moves
a trained classifier, without retraining, from a full-precision GPU to a
quantized CPU service, changing the device, numerical format, and operator
path together.
Figure~\ref{fig:migration} reports this combined change by classifier
formulation. The effect grows with how generative the classifier is. The
discriminative BERT classifier changes about 0.3 percent of labels, the two
pseudo-generative formulations about 2\% each, and the fully generative
GPT-2 classifier about 5.6 percent on average and up to 39 percent on its worst
checkpoint. We report this as a whole-migration profile rather than a single-factor
estimate; Appendix~\ref{sec:app-device-quant} separates the two factors.

A second combined change occurs after the deployment is fixed, because real
traffic arrives in large, variable-length batches that the serving system
batches and pads dynamically. Figure~\ref{fig:servingload} shows that this
serving load moves the fully generative classifier, about 1.6 percent of
labels at full precision and more under 8-bit integer execution, while it
leaves the discriminative classifier near the floor. The driver is the
realized padded width rather than the batch count.

\begin{figure}[t]
\centering
\begin{subfigure}[b]{0.48\linewidth}
\centering
\includegraphics[width=0.82\linewidth]{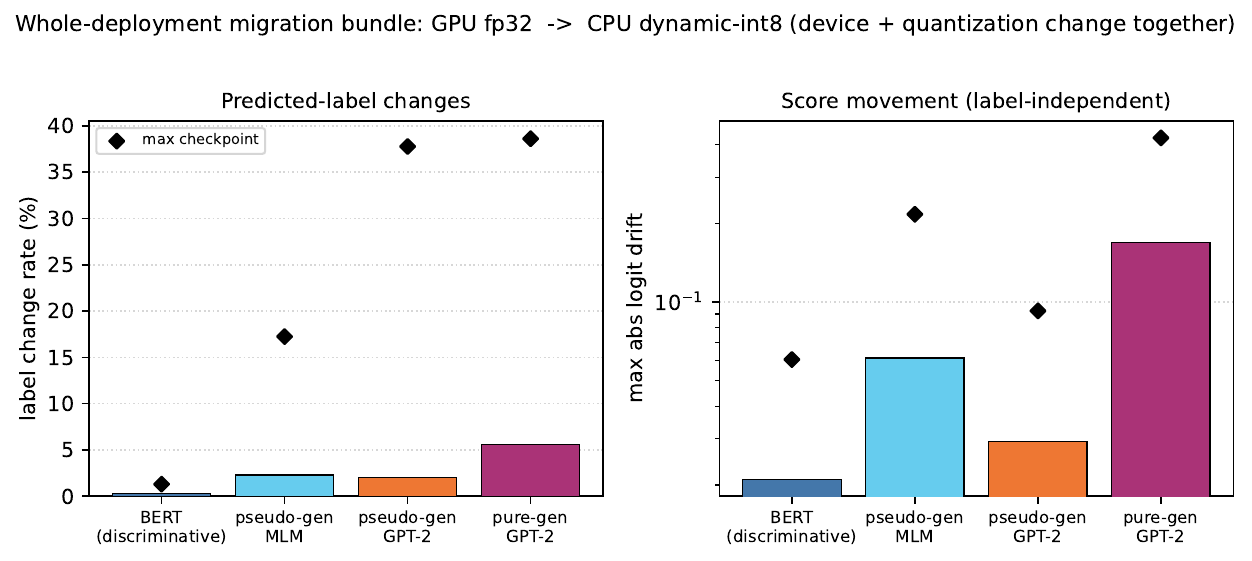}
\caption{Whole-deployment migration (full-precision GPU to quantized CPU),
by classifier formulation. Reported as a whole-migration profile, not a
single-factor estimate.}
\label{fig:migration}
\end{subfigure}\hfill
\begin{subfigure}[b]{0.48\linewidth}
\centering
\includegraphics[width=0.82\linewidth]{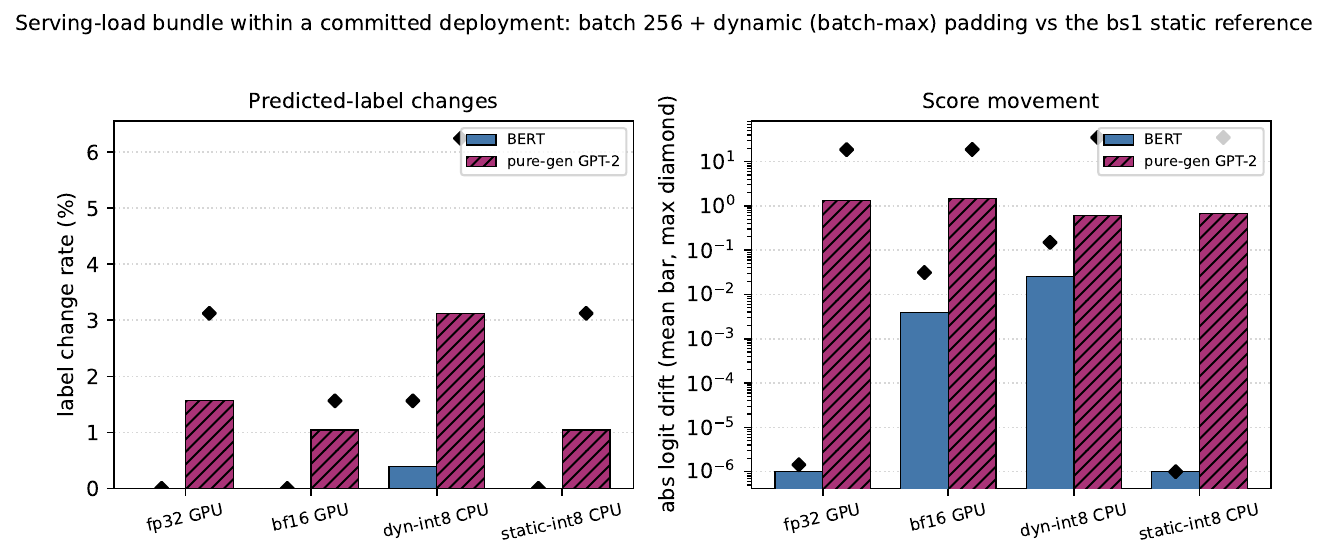}
\caption{Serving load within a fixed deployment, per deployment: a large
batch with dynamic (batch-maximum) padding against the single-example
static-width reference.}
\label{fig:servingload}
\end{subfigure}
\caption{Two combined deployment changes a practitioner encounters: a
whole-deployment migration (left) and serving load within a fixed
deployment (right).}
\label{fig:bundles}
\end{figure}

The two combined changes rank differently under labels than under scores, and
the reason applies whenever only the label is monitored. Holding the dataset
fixed, the migration and the serving load flip the same fraction of fully
generative labels. However, the serving load moves the raw scores several
times more in the mean and far more in the worst example, while the
probability movement stays comparable. A label change is governed by the gap between the two leading
classes after the bounded softmax. The small, near-uniform perturbations that the
quantization introduces therefore convert efficiently into flips near the
boundary. By contrast, the large, heavy-tailed score changes that the padding
mechanism produces often fall on confident examples, so they inflate the score
movement without changing the label. Comparisons
across the two changes must therefore hold the dataset and its number of
classes fixed, and a label-only monitor would rank them backward.

\subsection{Mitigations and their costs}
\label{sec:mitigations}

Each measured mechanism has a targeted mitigation, and none
covers the others (Table~\ref{tab:mitigations}). Assigning positions by
real-token order removes the structural padding effect, provided the same
rule is used in training, since applying it only at inference degrades
accuracy or collapses predictions (Appendix~\ref{sec:training-consistency}).
Because this rule no longer requires a fixed maximum width, it also permits
dynamic padding, which
processed about one quarter as many padded tokens for about two to two and a
half times the target throughput in our harness. Batch-invariant
matrix-multiplication kernels drive the batch-shape movement to bitwise
equality at full precision and bf16, and they remove the residual bf16 flips.
Their precision-matched cost is roughly neutral at batch size 1 and about
$1.3$ to $1.5\times$ the standard latency at large batch. Adopting them changes the numerical
result, so they are a deployment choice rather than a run-time switch. A fixed offline 8-bit integer scale removes the
coupling between co-batched inputs. For the small-margin flips that remain under
reduced precision, recomputing only the low-margin predictions in the reference
context detects almost all observed flips while recomputing a small fraction of
requests. All costs are measured in our experimental setup and do not represent production speedups.

\begin{figure}[t]
\centering
\begin{minipage}[t]{0.46\linewidth}
\vspace{0pt}
\centering
\includegraphics[width=0.86\linewidth]{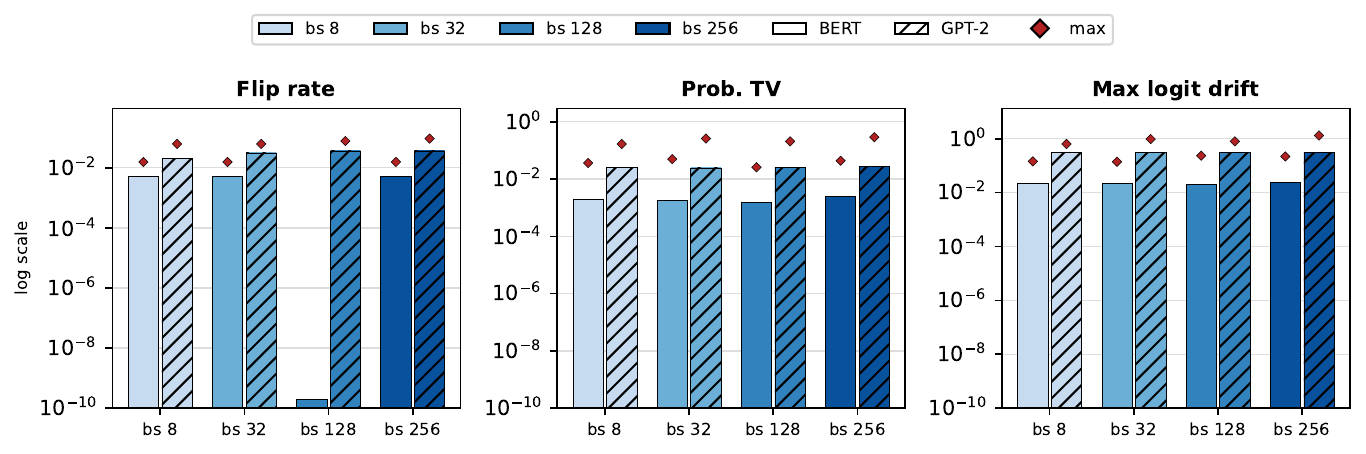}
\caption{Changing which other examples share the batch, under dynamic 8-bit
integer quantization on CPU, at a fixed batch shape, as a function of batch
size.}
\label{fig:v1comp}
\end{minipage}\hfill
\begin{minipage}[t]{0.5\linewidth}
\vspace{0pt}
\centering
\captionof{table}{Mechanism-matched mitigations. Costs are measured in our experimental setup and do not represent
production speedups.}
\label{tab:mitigations}

\tiny
\setlength{\tabcolsep}{3pt}

\begin{tabular}
{@{}p{0.3\linewidth}p{0.27\linewidth}p{0.37\linewidth}@{}}
\toprule
Mechanism & Mitigation & Measured outcome and cost \\
\midrule
Width-dependent positions & order-based positions & flips removed; $2.0$ to $2.5\times$ throughput \\
Batch-shape reductions & batch-invariant GEMMs & bitwise 64/64; $1.3$ to $1.5\times$ latency at large batch \\
Dynamic activation scale & fixed (offline) scale & dependence removed, bitwise (phone integer graphs and server control) \\
Residual small-margin flips & margin-based recompute & 44/44 held-out flips detected at 22\% recompute \\
\bottomrule
\end{tabular}

\end{minipage}
\end{figure}

\subsection{Formal results}
\label{sec:theory}

Three formal statements specify when the effects measured above can occur. Appendix~\ref{app:theory} states
the assumptions and gives the proofs.

The first statement identifies when padding can change a prediction. It gives a
sufficient condition for a change of padded width to leave the prediction
unchanged, together with the converse for the one scheme that violates it.

\begin{cthm}{Width invariance is governed by the position rule}{widthinv}
Consider two left-padded tensors containing the same real-token sequence,
differing only in the number of padding columns, so that every real token's
tensor column shifts by a common offset $s$. Suppose (i) every positional
contribution $\Phi$ to the computation is invariant to a common shift,
$\Phi(i+s, j+s) = \Phi(i, j)$; (ii) padding is excluded by the mask from
every attention and pooling operation; (iii) the output position from which the class scores are read follows the same
real-token state; (iv) any cache offsets use one consistent coordinate
system; and (v) no other operation depends on the absolute tensor column.
Then, in exact arithmetic, the classifier logits are identical at both
widths. Conversely, if positions index a learned absolute embedding table
by tensor column, the two widths present different represented
inputs whenever $p(\pi_i + s) \neq p(\pi_i)$ for any real token, and no
same-input numerical guarantee applies. Proofs:
Propositions~\ref{propx:common-shift} and~\ref{propx:absolute-position},
with the per-scheme instantiation in Proposition~\ref{propx:schemes},
Appendix~\ref{app:theory}.
\end{cthm}

With correct masking and a correct output position, whether padding changes the represented
input is therefore decided by the position rule, before any numerical
consideration. A rule that depends only on
the distance between tokens, or that numbers real tokens from zero, is safe.
A learned embedding indexed by the tensor column is not. This theorem gives a
sufficient condition together with its converse for the learned-absolute
case. It does not characterize every scheme. It also covers exact arithmetic
only, since removing the structural mechanism does not by itself guarantee
bitwise equality under finite-precision execution. Next, the following statement limits which predictions are influenced by a numerical change. It is a one-sided
certificate whose contrapositive is the necessary condition for a flip.

\begin{cprop}{Margin certificate for label stability}{margin}
Let $w$ be the unique predicted class under the reference context $c_0$,
with probability margin
$m(x) = p_w(x,c_0) - \max_{j \neq w} p_j(x,c_0)$, and let
$\operatorname{TV}(x,c,c_0) = \tfrac{1}{2}\sum_y \lvert p_y(x,c) -
p_y(x,c_0)\rvert$. If $m(x) > 2\operatorname{TV}(x,c,c_0)$, the predicted
class is identical under $c$ and $c_0$; equivalently, every label change
requires $m(x) \le 2\operatorname{TV}(x,c,c_0)$. The analogous logit
statement holds with the logit margin $g(x)$ and threshold $2\epsilon$ when
the logit change is bounded by $\epsilon$ in the infinity norm. Proofs:
Propositions~\ref{propx:logit-margin}, \ref{propx:tv-margin},
and~\ref{propx:flip-necessary} in Appendix~\ref{app:theory}.
\end{cprop}

A label can change only where the two leading classes are close, which is why
reduced precision and quantization move labels only near the decision boundary
while a wide margin leaves the label intact. The certificate is one-sided. It
preserves the label, not the calibration, ranking, or thresholds that consume the scores, and
it is sufficient rather than necessary for stability, so a label can also remain
fixed below the threshold. For a fully generative classifier that sums a score
over every token, the admissible movement grows with the input length, which is
why that formulation is the most exposed (Proposition~\ref{propx:gen-margin}).

The third statement states the single condition under which speed, accuracy, and
serving-context invariance cannot hold together.

\begin{cthm}{Conditional infeasibility of joint speed, accuracy, and invariance}{trinity}
Fix a task distribution, a bounded context set $\mathcal{C}$, one execution
stack, an implementation family $\mathcal{H}$, and predeclared thresholds
$\alpha$, $\tau$, $\delta$ on worst-context accuracy
$A(f) = \inf_{c} \Pr[\arg\max_y p_f(y \mid X,c) = Y]$, worst-case latency
$T(f)$, and score drift
$V(f) = \sup_{x,c,c'} \operatorname{TV}(x,c,c')$. Suppose a statistic
$Q = q(X)$ exists such that every implementation with $T(f) \le \tau$ and
$V(f) \le \delta$ predicts as a deterministic function of $Q$ in every
context. If the Bayes accuracy available from $Q$ satisfies
$A_Q^{\star} < \alpha$, then no implementation in $\mathcal{H}$ meets all
three thresholds. Proof: Theorem~\ref{thmx:triple} in
Appendix~\ref{app:theory}; Proposition~\ref{propx:collision} derives the
premise from an execution model and Proposition~\ref{propx:four-input}
gives a pairwise-feasible, jointly infeasible instance.
\end{cthm}

In the configurations we tested, none of the three had to be given up to
bound the others. Order-based positions removed the structural mechanism while
improving throughput, fixed 8-bit integer scales preserved quality, and
batch-invariant execution removed the batch-shape drift at the bounded cost
reported above. The theorem does not assert that such a tradeoff exists.
Its premise holds for execution models whose fast path exposes a fixed
representation, and no such model has been identified for these
classifiers. It marks the point at which being fast and invariant would
force a classifier to discard information that accuracy needs.
Propositions~\ref{propx:dynamic-scale}, \ref{propx:fixed-scale},
and~\ref{propx:gen-drift} state the quantization coupling of
Section~\ref{sec:request} and the length-proportional generative drift.





\section{Conclusion}

\textbf{Findings.} For the classifiers we study, a prediction is not fixed
by the trained model and the input text alone. It also depends on serving
settings the user never sees. The largest effect is input-side. When a
classifier with a learned embedding for each absolute position is padded on
the left, changing the padded length moves every real token to a new
position and changes the input the classifier receives. Assigning positions
by order among the real tokens, in training and inference alike, removes the
effect in our experiments without lowering accuracy. The numerical causes (reduced
precision, batch shape, quantization) are smaller and concentrate near the
decision boundary, yet a stable label can coexist with large probability movement
that thresholds, ranking, and filtering consume, and a wide margin does not
protect against the input-side cause.

\textbf{Scope.} The measurements cover transformer text classifiers on
unimodal text, trained from initialization in the from-scratch grid and
fine-tuned in the pretrained controls, on one GPU, one CPU, and two Android
phones. The mechanism statements do not depend on this choice, since width
invariance is decided by the position rule and the margin certificate
applies to any serving change. What remains to be measured for in-context
classifiers, multimodal inputs, and production engines is magnitude. The three model sizes scale depth, width,
and head count jointly and describe how the effect moves with capacity.
The GPU-to-CPU comparison is the one-step migration a practitioner
performs, with its components separated in Appendix~\ref{sec:app-device-quant}.
Most evaluation examples sit at wide margins. The five-class cells confirm
the certificate's prediction that narrower margins convert the same score
movement into more label changes (Table~\ref{tab:labelcount}).
Appendix~\ref{app:limitations} states the boundaries in full. 

\textbf{Recommendation.} Reproducible classification requires fixing the
serving settings we identify, monitoring class probabilities rather than
labels, and recomputing inputs on which fixed settings disagree.

\label{lastmainpage}

\subsubsection*{AI Use Statement}
Generative AI tools were used to assist with manuscript editing, code
implementation, and checking mathematical arguments and experiments. The
authors made all scientific decisions, verified the incorporated outputs,
and take responsibility for the final content.

\subsubsection*{Ethics Statement}
This work uses only public benchmark datasets (SST-2, SST-5, AG News,
Emotion, and IMDb) and collects no human-subject data. The measured
instabilities are relevant to deployed classification systems, and the
paper provides mitigations rather than mechanisms for exploitation.

\subsubsection*{Reproducibility Statement}
Every experiment is specified by a manifest that records the checkpoint,
dataset, random seed, serving settings, and output location, and every run
writes to its own output directory, so each reported number can be traced
to a stored artifact. Appendix~\ref{app:theory} provides the proofs,
Appendix~\ref{app:implementation} reports the model, training, and serving
settings, and Appendix~\ref{app:results} contains the complete result
tables and figures.

\bibliographystyle{iclr2027_conference}
\bibliography{references}

\appendix
\clearpage
\begin{center}\Large\bfseries Appendix Contents\end{center}
\startcontents[appendix]
\printcontents[appendix]{l}{1}{\setcounter{tocdepth}{2}}
\bigskip

\section{Notation}
\label{app:notation}

\begin{table}[h]
\centering
\small
\begin{tabular}{ll}
\toprule
Symbol & Meaning \\
\midrule
$x$; $\theta$ & input text; trained checkpoint \\
$c$; $\mathcal{C}$; $c_0$ & serving context; declared context set; reference context \\
$f_\theta(x;c)$ & predicted label of checkpoint $\theta$ on $x$ under context $c$ \\
$p_\theta(\cdot \mid x, c)$; $z$ & class-probability vector; class logits \\
$\mathrm{Flip}_\theta(x; c_0, c)$ & indicator that $f_\theta(x;c) \neq f_\theta(x;c_0)$ \\
$\operatorname{TV}(x,c,c')$ & total variation between the probability vectors under $c$ and $c'$ \\
$\lVert z_0 - z_c \rVert_\infty$ & maximum absolute logit difference between contexts \\
$m(x)$; $g(x)$ & top-two probability margin; top-two logit margin, both under $c_0$ \\
$L$; $T$; $s$ & real-token length; padded width; common column shift from added padding \\
$\pi_i$; $p(\cdot)$ & tensor column of real token $i$; learned absolute position embedding \\
$\Phi(i,j)$ & positional contribution to the interaction of positions $i$ and $j$ \\
$A(f)$; $T(f)$; $V(f)$ & worst-context accuracy; worst-case latency; worst-case score drift \\
$Q = q(X)$ & statistic through which fast and invariant implementations predict \\
\bottomrule
\end{tabular}
\end{table}

\section{Additional Results}
\label{app:results}

This appendix collects the measurements that the main results reference,
organized as listed in the contents above.

\subsection{Identical contexts are bitwise repeatable}
\label{sec:floor}

Every effect reported below is measured against repeated identical runs.
Repeating the identical serving context produces bit-identical logits on
every measured configuration (Table~\ref{tab:floor}). Repetition contributes exactly zero, so every
nonzero difference reported below is attributable to the declared
intervention. This also separates our subject from run-to-run
nondeterminism, since the models are deterministic given a serving context
but not invariant across serving contexts.

\begin{table}[t]
\centering
\caption{Same-context repeatability floor. Identical inputs, identical
context, repeated runs; entries are bitwise-identical comparisons over the
total.}
\label{tab:floor}
\small
\setlength{\tabcolsep}{3pt}
\begin{tabular}{lll}
\toprule
Configuration & Conditions & Result \\
\midrule
CPU PyTorch fp32 & BERT and GPT-2; 1 and 32 threads & bitwise (4 repeats) \\
GPU PyTorch & fp32, fp16, bf16; all four formulations & bitwise (4 repeats) \\
GPU PyTorch, held-out replication & fp32, bf16, fp16 & 2{,}304 / 2{,}304 bitwise \\
ONNX Runtime CPU & BERT SST-2, five repeats & 64 / 64 bitwise \\
GPU fp32, no positional encoding & four identical contexts & 768 / 768 bitwise \\
GPU, fixed-width length comparisons & IMDB and AG News; fp32 and bf16 & 768 / 768; 576 / 576 bitwise \\
GPU, greedy decoding & batch 1 and 32, 3 prompts, 2 precisions & 12 / 12 checks bitwise \\
\bottomrule
\end{tabular}
\end{table}

\subsection{Position-assignment intervention}

Figure~\ref{fig:intervention} accompanies the position-assignment experiment
of Section~\ref{sec:position-intervention}, in which removing the
width-dependent positions eliminates the effect and deliberately restoring
them brings it back.

\begin{figure}[t]
  \centering
  \includegraphics[width=\linewidth]{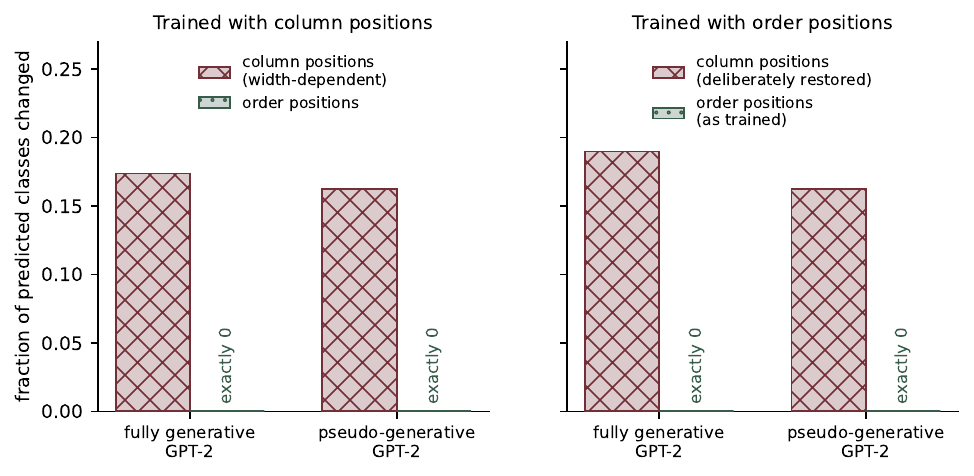}
  \caption{Intervening on the position assignment. Assigning positions by
  order among real tokens removes the observed padding-induced class
  changes; deliberately restoring column positions on models trained with
  order positions brings the changes back. Matching the assignment rule
  between training and inference is separately required to preserve
  accuracy (Section~\ref{sec:training-consistency}).}
  \label{fig:intervention}
\end{figure}

\subsection{High confidence does not guarantee stability}
\label{sec:margins}

\begin{figure}[t]
\centering
\begin{subfigure}[b]{0.49\linewidth}
\centering
\includegraphics[width=\linewidth]{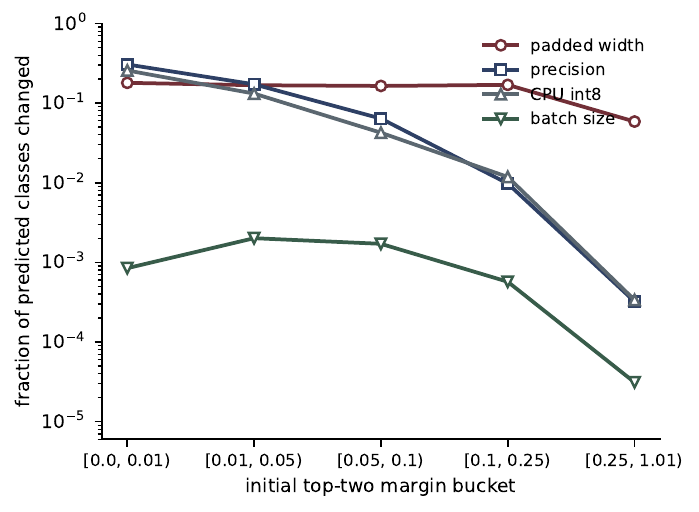}
\caption{Fraction of predicted classes changed, by initial top-two logit
margin in fixed intervals $[0,0.01)$, $[0.01,0.05)$, $[0.05,0.1)$,
$[0.1,0.25)$, and $[0.25,\infty)$ (95\% intervals clustered by model). Numerical changes concentrate where the two leading
scores are nearly tied; padding persists to wider gaps. Zero-observed groups
sit at the $10^{-5}$ floor.}
\label{fig:margin}
\end{subfigure}\hfill
\begin{subfigure}[b]{0.49\linewidth}
\centering
\begin{tikzpicture}
\begin{axis}[
    width=\linewidth, height=0.8\linewidth,
    xmode=log, log basis x=2, ymode=log,
    xlabel={real-token length (width 256)},
    ylabel={max abs.\ logit difference},
    xtick={32,64,128,256}, xticklabels={32,64,128,256},
    legend style={at={(0.98,0.02)},anchor=south east,font=\scriptsize,draw=none,fill=none},
    tick label style={font=\scriptsize}, label style={font=\footnotesize}]
\addplot[color=siwine,mark=*,thick] coordinates
    {(32,0.406) (64,0.636) (128,0.677) (256,1.530)};
\addplot[color=sinavy,mark=square*,thick] coordinates
    {(32,6.10e-5) (64,1.22e-4) (128,1.83e-4) (256,2.44e-4)};
\legend{bf16, fp32}
\end{axis}
\end{tikzpicture}
\caption{Within-example length intervention (IMDB, fully generative GPT-2,
batch size 1 versus 32, 256 targets per point, bitwise floor). bf16 label
changes rise 2, 6, 6, 7 of 256 with length; fp32 changes none.}
\label{fig:length}
\end{subfigure}
\caption{Two further amplifiers of numerical drift: the decision margin
suppresses numerical label changes but not padding (left), and a longer
real-token input increases per-forward drift at fixed padded width (right).}
\label{fig:secondary-pair}
\end{figure}

Higher predicted-class probability is associated with fewer label changes but
not with less score movement. Across the matched sizes of
Table~\ref{tab:size}, the average probability assigned to the predicted class
rises from 0.554 to 0.846, while class
changes from precision fall from 3.68\% to 0.568\% and from the GPU-to-CPU
comparison from 3.30\% to 0.666\%; both contrasts exclude zero. The
underlying score movement from precision nevertheless rises by 0.0118 (95\%
CI 0.0076 to 0.0160); across these models, higher confidence coincides with
fewer label changes despite larger score movement, a comparison that does
not separate confidence from model size or training quality. Figure~\ref{fig:margin} shows where
the changes occur. For the two numerical comparisons they concentrate where
the two leading scores are nearly tied, whereas padding continues to change
classes at score gaps where the numerical comparisons show none. Where the top-two logit margin is at least 4, padding still changes 51 of 46,006 binary decisions and 312 of 74,745 multiclass decisions, while precision and the GPU-to-CPU comparison change none across 3.65 million comparisons. Confidence
therefore suppresses label changes from the numerical effects, which perturb the computation of a
fixed input, but not the padding effect, which changes the input itself.
These observations do not establish confidence as the cause, and
a probability threshold calibrated under one inference setting carries no
guarantee under another.

\subsection{Computational length amplifies numerical drift}
\label{sec:length}

A longer target amplifies the numerical mechanism even when nothing structural
changes. Truncating the same 256 IMDB reviews to real-token lengths 32 to
256 at a fixed padded width of 256 increases drift monotonically with
length, and the trend holds within a fixed margin band
(Figure~\ref{fig:length}). At bf16 the maximum logit difference grows from
0.41 to 1.53 and flips from 2 to 7 of 256, while fp32 stays at the
$10^{-4}$ scale with zero flips at every length. A second dataset (AG News,
fixed width, bins of different targets) reproduces the descriptive contrast between fp32 and
bf16 (2/256, 0/256, and 2/54 flips).

\begin{cprop}{Length-proportional drift bound for generative scoring}{length}
For class $y$, let $S_y(c) = \sum_{t=1}^{L} \ell_{y,t}(c)$ be the summed
token log-probability score over the $L$ real tokens under context $c$,
padding excluded. If every token contribution satisfies
$\lvert \ell_{y,t}(c) - \ell_{y,t}(c') \rvert \le \varepsilon$, then
$\lvert S_y(c) - S_y(c') \rvert \le L\varepsilon$ for each class, and every
top-two score gap changes by at most $2L\varepsilon$ (or
$(L_y + L_j)\varepsilon$ under class-dependent lengths), against
$2\varepsilon$ for a classifier that reads a single output position. If the reference score margin
exceeds this bound, the prediction cannot change. Proofs:
Propositions~\ref{propx:gen-drift} and~\ref{propx:gen-margin} in
Appendix~\ref{app:theory}.
\end{cprop}

A classifier that scores every token of the input under
each class accumulates a perturbation that grows with the number of tokens it
reads, while a classifier that reads a single output position does not. This is the aggregation mechanism consistent with the measured
difference between formulations and with the drift trend in Figure~\ref{fig:length}; the
bound limits how much drift the aggregation can admit; it does not predict
the realized drift or a corresponding change in label stability.

\subsection{Score movement under padding grows with model size}
\label{sec:size-drift}

Table~\ref{tab:size} compares models of three sizes, restricted to
combinations in which all three sizes pass the accuracy threshold. The
fraction of predicted classes changed by padding shows no consistent size
trend (25.0\%, 18.4\%, 26.4\% for small, medium, and large; the
large-minus-small interval crosses zero). Score movement under unchanged
predictions increases with model size. Among examples whose predicted class does
not change, the average total variation distance rises from 0.0716 to
0.0880 to 0.1441, and the largest score change rises from 5.09 to 10.06 to
20.84. The paired large-minus-small differences are 0.0725 (95\% CI 0.0039
to 0.1386) and 15.75 (9.27 to 23.41). Larger models therefore do not become
more stable under padding. Their scores move at least as much under equally
stable labels. Because size, width, depth, and training quality change
together, these are descriptive associations, not a causal effect of size.

\begin{table}[t]
\centering
\scriptsize
\caption{Comparison across model sizes, restricted to models of comparable
accuracy. Padding rows use the 16 GPT-2 model, dataset, and seed
combinations for which all three sizes pass the accuracy threshold; the
remaining rows use 40 matched combinations across all four families. Total
variation (TV) is computed only over examples whose predicted class does
not change. Paired uncertainty for the principal contrasts is reported in
the text.}
\label{tab:size}
\setlength{\tabcolsep}{3pt}
\begin{tabular}{@{}llrrr@{}}
\toprule
Axis & Metric & Small & Medium & Large \\
\midrule
Padding & Flip & 0.2505 & 0.1840 & 0.2642 \\
Padding & TV (no flip) & 0.0716 & 0.0880 & 0.1441 \\
Padding & $\lVert\Delta z\rVert_\infty$ (no flip) &
5.09 & 10.06 & 20.84 \\
\midrule
Precision & Flip & 0.0368 & 0.0103 & 0.00568 \\
Deploy. profile & Flip & 0.0330 & 0.00879 & 0.00666 \\
Batch & Flip & $7.5{\times}10^{-6}$ & $1.09{\times}10^{-5}$ &
$3.56{\times}10^{-4}$ \\
Batch & TV (no flip) & $2.26{\times}10^{-6}$ & $4.0{\times}10^{-6}$ &
$2.32{\times}10^{-4}$ \\
\bottomrule
\end{tabular}
\end{table}

\subsection{Batch size moves scores even where predicted classes hold}
\label{sec:batch-scores}

Class changes from batch size remain rare at every model size, but score
movement among examples whose predicted class does not change grows. Among examples whose predicted class does not
change, the average total variation distance rises from $2.26\times10^{-6}$
in the small models to $2.32\times10^{-4}$ in the large ones, a paired
increase of $2.29\times10^{-4}$ (95\% CI $1.15\times10^{-4}$ to
$3.71\times10^{-4}$). These comparisons pool more than one numerical
precision, so they do not measure a pure fp32 trend. They show that a
tolerance on predicted classes and a tolerance on scores are separate
requirements.

\subsection{Matching the position rule between training and inference}
\label{sec:training-consistency}

When a model trained with column positions is served with order positions,
the accuracy cost reaches 9.88 points, because that model receives position
inputs it was never trained on. When each position rule is instead
used consistently in both training and inference, across three seeds and
five datasets, the accuracy differences between the two rules are $-0.15$
and $+0.32$ percentage points at the family level, with conservative 95\%
intervals inside a $\pm2$ percentage-point band. For a learned absolute
embedding, order positions under left padding produce the same position
representation as right padding, so this consistently trained
configuration is the right-padded baseline. Order positions
additionally allow each batch to be padded only to its own longest sequence
rather than to a fixed maximum width. A pilot on one seed processes 73.5\%
fewer padded token slots. Its latency ratios are measurements of our
evaluation setup, not production speedups.

\subsection{Padding across architectures: further detail}
\label{sec:app-arch}

\begin{table}[t]
\centering
\small
\caption{Padding-induced class changes by position mechanism and padding
side. Changes occur only where padding moves real tokens that carry a
learned embedding for each absolute position.}
\label{tab:architecture}
\begin{tabular}{lll}
\toprule
 & Learned absolute & Rotary (RoPE) \\
\midrule
Left-padded & GPT-2: 35{,}591 / 176{,}778 (20.1\%) & Qwen3: 0 / 29{,}463 \\
Right-padded & BERT, MLM: 0 / 176{,}778 & ModernBERT: 0 / 29{,}463 \\
\bottomrule
\end{tabular}
\end{table}

Fine-tuning public pretrained GPT-2 models reproduces the effect at larger
magnitude. The mean class-change rate across its 15 dataset and seed
combinations is 40.015\% (95\% confidence interval 20.620 to 53.790\% from
a hierarchical bootstrap), or 11,780 of 29,463 individual evaluations. All
15 of these models pass the accuracy threshold and predict more than one
class, although the SST-5 models are weak.

\paragraph{Zero disagreement is not sufficient.}
One correction leaves training unchanged and, only at
inference, assigns each real token a position index equal to its order among
the real tokens, which no longer depends on the padded width. Under this
change the padding effect nearly vanishes (one changed class in 29,463
evaluations), but five of the fifteen models begin predicting the same
class for every input. A model with a constant prediction is perfectly
stable and uninformative, so its zero disagreement has no evidential weight.
We therefore use this experiment only as further evidence that pretrained
GPT-2 fine-tunes are vulnerable to padding. Any proposed correction must
pass three checks at once. The effect must disappear, the predictions must
remain accurate and varied, and a deliberately vulnerable control must
still show the effect, which confirms that the measurement could have
detected it.

A classifier trained from initialization with no position encoding at all
provides the boundary case. It reaches accuracy 0.747, its padded width
changes on all 256 evaluated examples, and none of its predicted classes
change, with residual score movement only at the fp32 numerical scale (at
most $7.6\times10^{-6}$ in the logits). Removing positional encoding
removes the structural mechanism entirely, while the numerical effect of
execution shape remains.

\subsection{Position-assignment mechanism: further detail}
\label{sec:app-position}

Condition (i) sorts the known position-encoding schemes. It is satisfied by
positions assigned by order among real tokens (which do not move at all),
and by every scheme whose positional term depends only on the difference
$j - i$: rotary embeddings, ALiBi's linear distance bias, T5's bucketed
relative bias, Shaw's clipped relative representations, and Transformer-XL
relative attention. A model with no position term satisfies it vacuously.
It is violated by learned absolute embeddings indexed by tensor column, and
sinusoidal absolute encodings offer no general end-to-end cancellation once
the encoding is mixed into content by learned projections. With correct masking and a correct output position, whether padding changes the
represented input is decided by the position rule, before any numerical
consideration, and a serving system that
wants padding freedom must either use a difference-based rule or assign
positions by order among real tokens. Of these schemes, we
directly measure learned absolute (affected), order positions, rotary
(Qwen3 and ModernBERT), and a native no-position model (all unaffected).
The ALiBi, T5, Shaw, Transformer-XL, and sinusoidal rows are
literature-supported mechanism boundaries, not measurements from our experiments.

The intervention is plotted in Figure~\ref{fig:intervention}
(Appendix~\ref{app:results}). On conventionally trained GPT-2 models,
switching inference to order positions
reduces family-mean class-change rates from 16.25\% and 17.37\% to zero
observed. We then train models with order positions and repeat the
comparison across three seeds, two GPT-2 families, and five datasets.
Served with the same order positions they were trained with, these models
show 0 changed classes in 25,986 comparisons in which the padded width
changes. Deliberately reassigning column positions to the same models
restores the effect. 5,769 of 25,986 classes change, with every family and
seed average between 17.4\% and 26.9\%. This deliberate restoration rules
out a false negative. A broken measurement that made the two padding
conditions identical could not reproduce the effect on demand.

We reproduce the same removal and restoration on CPU. Under column positions,
23.4\% of predicted classes change on SST-2 and 7.8\% on AG News when the
padded width changes, and under order positions none change on either
dataset. The effect is therefore a property of the position rule, not of
one device configuration.

\subsection{Further detail on the numerical mechanism}
\label{sec:app-numerical}

The bf16 label changes concentrate at small decision margins
(Figure~\ref{fig:margin}). The change rate falls monotonically from 30.6\%
below margin 0.01 to 0.18\% at margin 0.25 or greater, a factor of about
167. Margin preserves the label under numerical perturbations, but it is
not a safety certificate for scores. The single largest probability movement
(0.567) occurs in the widest margin group, where no label changes. Label
count converts drift into flips through the same margin channel
(Table~\ref{tab:labelcount}). On five-class SST-5 the same bf16 batch
intervention flips 6.2 to 7.8\% of fully generative predictions at roughly
half the logit drift of binary SST-2 (about 3.1\%). With more classes the
probability mass is divided among more alternatives, so the top-two margin
is smaller on average and the same logit movement crosses the decision
boundary for a larger fraction of inputs.

\begin{table}[t]
\centering
\small
\caption{Label changes and logit drift for the fully generative GPT-2
classifier under the same bf16 batch-shape change, by number of classes
(64 targets per cell).}
\label{tab:labelcount}
\begin{tabular}{lcc}
\toprule
Dataset & Label changes & Max.\ logit drift \\
\midrule
SST-2 (2 classes) & 2 of 64 (3.1\%) & about 0.45 \\
SST-5 (5 classes) & 4 to 5 of 64 (6.2 to 7.8\%) & 0.21 to 0.27 \\
\bottomrule
\end{tabular}
\end{table}

A prediction can only change if its confidence gap is
at most twice the probability movement the serving change produces. Label
changes are therefore confined to examples whose top-two margin is within
that bound, which for the small numerical perturbations measured here means
nearly tied predictions, and a wide margin preserves the label without
limiting how far the scores may move. All 375 observed bf16 batch-induced label
changes satisfied the necessary condition (recall 100\%, precision 56.5\%),
whereas a margin threshold alone is a weak predictor (95\% recall requires
flagging every margin below 0.46). The certificate is one-sided. It bounds
labels, never scores.

\subsection{Quantization lets other examples in the batch change a prediction}
\label{sec:quantization}

We call the one example whose prediction we track the target. Under floating
point, the other examples in its batch cannot reach the target's prediction.
At fixed batch shape, their content leaves the target's output unchanged
bit-for-bit (1{,}080 GPU
composition pairs and all CPU cells exactly zero), because no operation
mixes one example's rows into another's and the kernel is chosen by tensor
shape, not by the values in it. Dynamic int8 quantization breaks this, because the
scale that converts activations to integers is computed from the whole batch
tensor, so the values of the other examples in the batch enter the target's own
computation
(Table~\ref{tab:int8}). Replacing the runtime scale with a fixed
offline-calibrated scale restores bitwise equality, and a fake-quantization
control that varies only the scale source reproduces the effect, isolating
the mechanism.

\begin{table}[t]
\centering
\caption{Effect of changing the other examples in the target's batch at
fixed batch shape (CPU, static width 128, 64 targets, SST-2 and AG News,
threads 1 and 32). Only the batch-global dynamic activation scale lets those
other examples reach the target.}
\label{tab:int8}
\small
\begin{tabular}{lcccc}
\toprule
 & fp32 & dynamic int8 & dynamic fake-quant & static int8 \\
\midrule
BERT max $|\Delta|$ logit & 0 (bitwise) & 0.09 & 0.015 to 0.066 & 0 (bitwise) \\
Gen.\ GPT-2 max $|\Delta|$ logit & 0 (bitwise) & 0.774 & 0.091 to 0.307 & 0 (bitwise) \\
Max TV & 0 & 0.124 & 0.092 & 0 \\
Label changes & 0 & up to 2 / 64 & up to 2 / 64 & 0 / 64 \\
Target slot (fixed shape) & bitwise & bitwise & --- & bitwise \\
\bottomrule
\end{tabular}
\end{table}

\begin{cprop}{Batch-global scales couple co-batched inputs; fixed scales do not}{int8}
Let $Q_s(a) = s\,\operatorname{clip}(\operatorname{round}(a/s), -M, M)$ be
the dequantized activation quantizer with scale $s$. (i) If the scale is
computed from the whole batch, $s = s(B)$, and two same-shape batches $B$,
$B'$ with the same target activation $a$ but different other examples
satisfy $s(B) \neq s(B')$ and $Q_{s(B)}(a) \neq Q_{s(B')}(a)$, then the
target's quantized activation depends on the other examples in the batch,
and its logits differ whenever the downstream map separates the two
activations. (ii) If instead the scale is fixed before serving and does not
depend on the other examples, and the target's execution otherwise does not
depend on them and is deterministic at fixed shape, then the other examples
cannot alter the target's output. Proofs:
Propositions~\ref{propx:dynamic-scale} and~\ref{propx:fixed-scale} in
Appendix~\ref{app:theory}.
\end{cprop}

Under a batch-global dynamic scale, the content of the
requests that share a batch sets the rounding grid applied to the target,
and fixing that grid offline removes the dependence. The result does not
extend to dynamic quantizers whose scale depends only on the target
request. Both directions match the measurements. Dynamic scales
moved the target by up to 0.774 logits with label changes, and fixed scales
were bitwise identical under the identical intervention.

Under dynamic int8 quantization on CPU, changing only the co-batched
examples at a fixed batch shape moves GPT-2 labels, and the effect grows
with the batch. The GPT-2 change rate rises from 8 of 384 comparisons
(2.08\%) at batch size 8 to 10 of 384 (2.60\%) at batch size 64 and
14 of 384 (3.65\%) at batch size 256, with the largest absolute logit
movement of 1.32 occurring at batch size 256. Each co-batched example
contributes to the batch-global activation scale, so a larger batch gives
the other examples more ways to change the quantization step applied to the
target. BERT shows almost no movement under the same change, with at most 1
change out of 192 comparisons in any single cell.

At the reference request context of batch size 1 with static padding there
is no batching and no other example in the batch, so switching the int8
pipeline from static to dynamic activation scales changes only the scale
computation. Under this change GPT-2 labels change on 4 of 192 targets
(2.08\%), all four in the small checkpoint (4 of 64), with a maximum
absolute logit movement of 0.839, while BERT changes on 0 of 192 targets
(Figure~\ref{fig:v2deploy}). Because no other example is present, none of
this movement can come from batch composition. It measures the scale policy
itself, separately from the coupling above.

\subsection{The target slot and the matrix-multiplication reduction order}
\label{sec:app-slot}

Moving the target to a different row at a fixed batch shape leaves the output unchanged in exact
arithmetic. A linear layer computes each output coordinate as a sum over the
hidden dimension, $y_n = \sum_k a_k W_{k,n}$, so the target row depends only on
its own activations. A batched matrix multiplication folds the batch and the
sequence into a single dimension $M$ of size $\text{batch}\times\text{sequence}$
and tiles the $M \times N$ output into blocks, partitioning the reduction over
the hidden dimension across those blocks. Which tile a row occupies, its
position inside that tile, and the partition boundary it falls on are all
functions of the row's offset in $M$, so moving the target to a different slot
changes the order in which its partial sums are accumulated. Floating-point
addition is not associative, so the rounded logits differ by about $10^{-5}$
for the fully generative GPT-2 classifier and about $10^{-6}$ for BERT, and the
gap grows mildly with batch size as a larger $M$ admits more tiles. The integer
CPU path uses one fixed reduction order for a given problem shape regardless of
the offset, so changing the target slot leaves the output bitwise identical there. An operator census of the forward
pass finds the linear-layer matrix multiplication is the only computation whose
reduction order depends on this offset. Attention, normalization, and the
softmax are computed per token and do not. The effect never crosses a decision
boundary, and the batch-invariant matrix-multiplication kernels of
Section~\ref{sec:mitigations}, which fix the reduction order independently of
$M$, remove it.

\subsection{Autoregressive feedback after the first token disagreement}
\label{sec:accumulation}

The batch-shape perturbation above never flips a high-margin classification
label, yet the generation literature reports large output changes from the
same root cause. The cause is autoregressive feedback, not the per-step
perturbation. Under greedy decoding with full recomputation and positions assigned by order
among the real tokens, with batch size 1 versus 32 as the only difference
between contexts, bf16 flips a greedy token within 4 to 14 steps in 3 of 3
margin-selected prompts, after which the trajectories separate (91 to 119 of
128 tokens differ). By contrast, fp32 diverges on none of 3 within 128 steps
(Table~\ref{tab:accum}). 
\begin{table}[t]
\centering
\caption{Accumulation experiment: greedy decoding under two batch shapes (GPT-2,
128 steps, 12 of 12 same-context floor checks bitwise).}
\label{tab:accum}
\small
\setlength{\tabcolsep}{4pt}
\begin{tabular}{lrrrr}
\toprule
Prompt (fp32 margin) & fp32 divergence step & bf16 divergence step & TV at divergence & mismatches at 128 \\
\midrule
Low (0.017) & none & 14 & 0.122 & 91 \\
Medium (0.579) & none & 4 & 0.094 & 119 \\
High (2.200) & none & 9 & 0.123 & 116 \\
\bottomrule
\end{tabular}
\end{table}

One classification forward reads a single, usually
high-margin argmax and never feeds its prediction back into the model. Greedy generation conditions every
later step on the decoded prefix, so the first low-margin step at which
the two contexts select different tokens changes all subsequent inputs. Initial
margin did not order the divergence step, so margin at the prompt is not a
schedule for when a generation diverges.

\subsection{Mitigation costs and backend comparisons}

Each measured mechanism admits a mitigation aimed at exactly that mechanism
(Table~\ref{tab:mitigations}). None of these mitigations covers the other mechanisms.
Order-based position assignment, which does not depend on the padded
width, eliminates the structural mechanism
and simultaneously permits dynamic padding, which processed $0.265\times$ as
many padded tokens for $2.04\times$ to $2.51\times$ measured target
throughput in our evaluation setup. Batch-invariant matrix-multiplication
kernels drive batch-shape drift to bitwise equality (64 of 64 targets) at
fp32 and bf16 for both formulations, removing the two bf16 label flips.
Their precision-matched end-to-end cost is roughly neutral at batch size 1 (a
dispatch-bound eager-mode effect) and about $1.3\times$ to $1.5\times$ at large
batch in true fp32, up to $1.45\times$ TF32-matched and $1.40\times$ at bf16
(Figure~\ref{fig:invkernel}). They cover the
linear-layer GEMMs plus mean and log-softmax, and adopting them changes the
computed values. Moreover, the released library computes fp32 matrix multiplications at
reduced tensor-core precision, so a comparison against a
full-precision standard baseline is not matched in precision unless the kernel is patched. Standard
versus invariant at a fixed batch size is never bitwise and moves up to 0.219
logits with 2 of 64 flips at bf16 on the generative head. Static calibration removes this dependence between co-batched inputs
(Table~\ref{tab:int8}). For residual numerical flips, a verifier that
recomputes only the small-margin predictions in a canonical context detects
95.2, 99.2, and 100\% of the 375 observed flips at 7.1, 10.8, and 13.8\%
recompute rates when thresholds are fitted on the same comparisons; with
thresholds fixed on a disjoint calibration split, the label policy
recomputes 22.45\% of held-out requests and detects all 44 observed
disagreements, while a fixed 0.8 threshold policy leaves 3 of 21 action
disagreements undetected (Section~\ref{sec:app-replay}). Backend changes
are label-stable but not score-stable in our cells. Eager versus
torch.compile moves logits by up to $3.1\times10^{-5}$ with 0 of 64
changes, PyTorch versus ONNX Runtime by up to $8.4\times10^{-7}$ with 0 of
64, CUDA Graph replay is bitwise at each captured shape while distinct
shapes retain ordinary batch drift, and semantic-position-preserving width
buckets remove the structural mechanism but keep numerical shape dependence
(bf16 generative max TV 0.044, 0 of 64 flips).

\begin{figure}[t]
\centering
\includegraphics[width=0.95\linewidth]{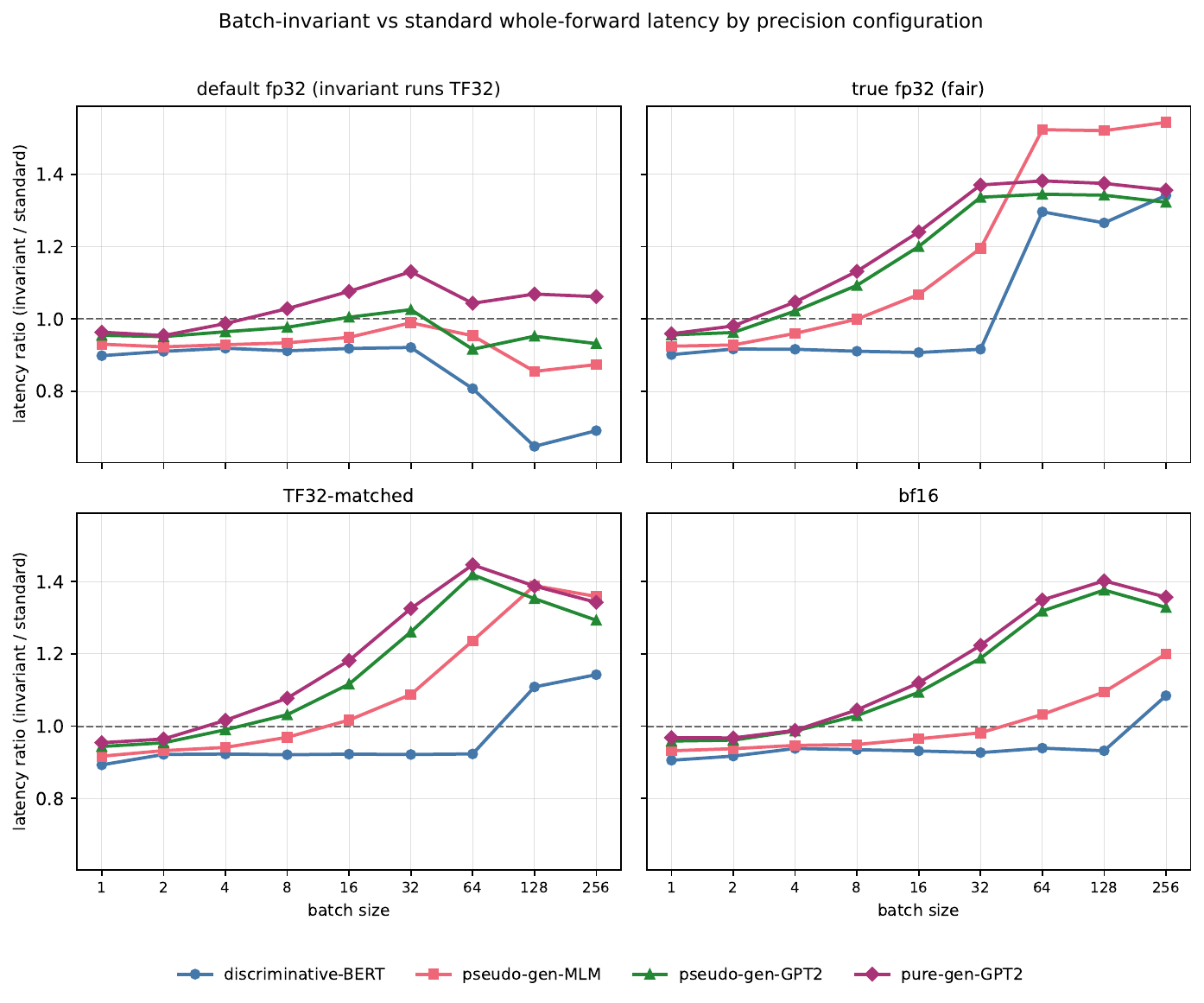}
\caption{Latency of batch-invariant matrix-multiplication kernels, as the ratio
of the batch-invariant to the standard whole-forward median, by batch size, for
four classifier formulations. A ratio below one is faster. Batch invariance is
roughly neutral at batch size 1 (a dispatch-bound effect under eager execution)
and costs about 1.3 to 1.5 times at large batch under a precision-matched
comparison (true fp32, TF32-matched, or bf16); the default fp32 comparison, in which
the invariant kernel runs at reduced tensor-core precision against a
full-precision baseline, is the only one that appears faster at large batch.}
\label{fig:invkernel}
\end{figure}

\subsection{Request-level factors by deployment}
\label{sec:app-v1full}

The main-text figures summarize the request-level factors on the
checkpoints that pass the accuracy threshold for which every request-level cell completed.
Figure~\ref{fig:app-knobs} compares batch sizes and padded widths for
each of the three deployments, and Figure~\ref{fig:app-within} gives the
within-batch rearrangements (target slot, order of the other examples, and
which examples share the batch) for each deployment. The dynamic-padding
versions of the within-batch figures are identical by construction, because the
padded width is held fixed under these rearrangements, and are omitted.

\begin{figure}[t]
\centering
\begin{subfigure}{0.32\linewidth}\centering
\includegraphics[width=\linewidth]{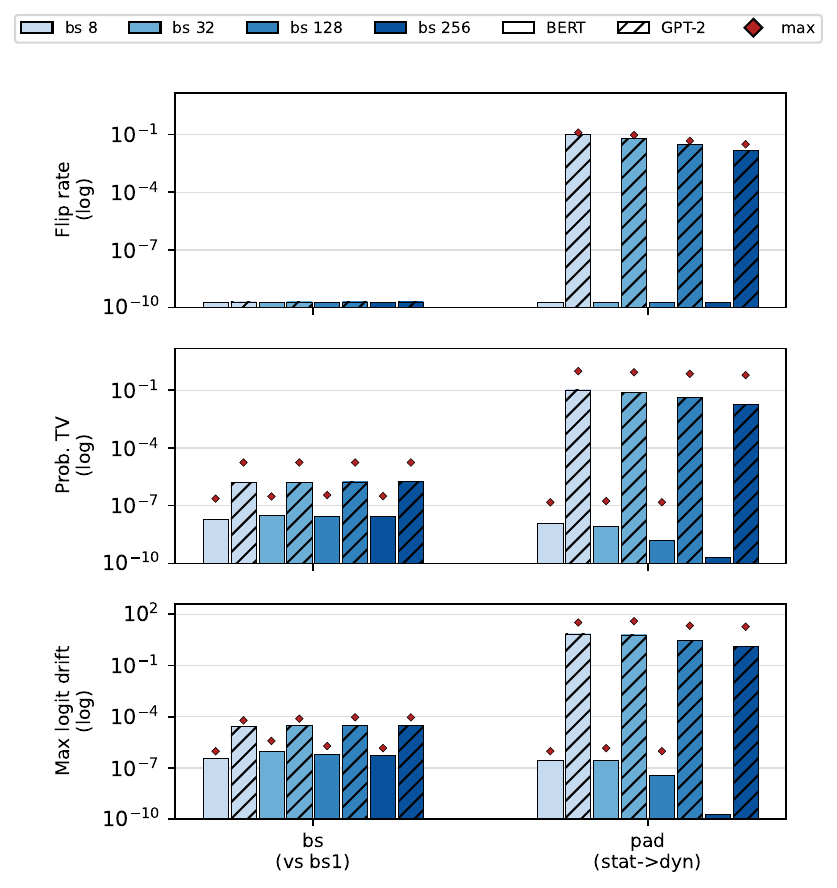}
\caption{full precision, GPU}\end{subfigure}\hfill
\begin{subfigure}{0.32\linewidth}\centering
\includegraphics[width=\linewidth]{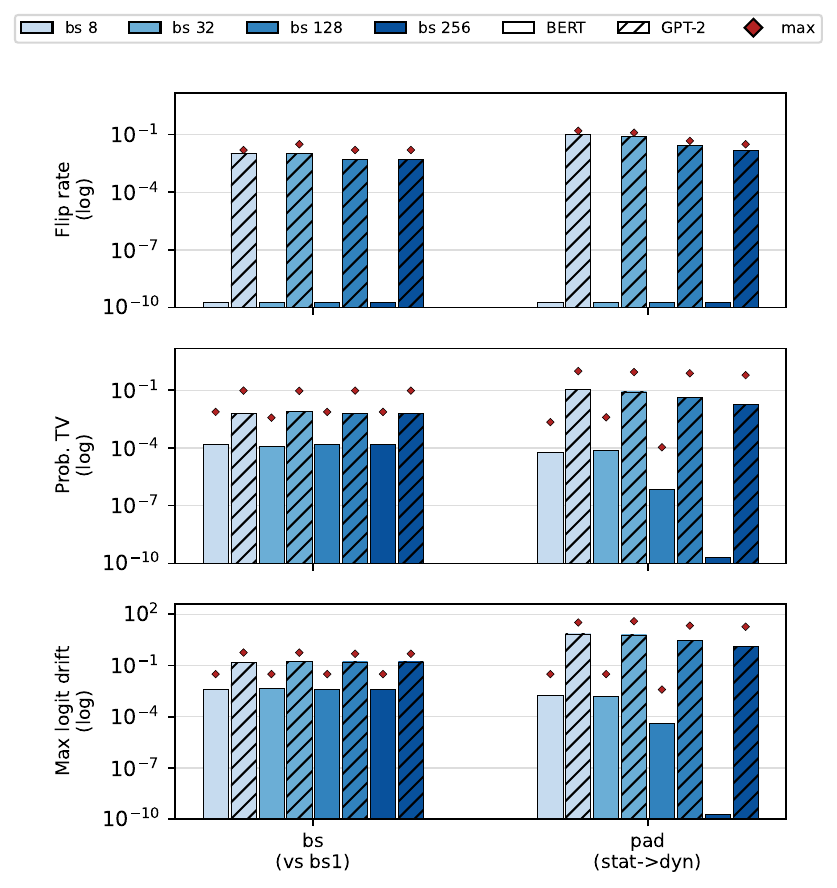}
\caption{bf16, GPU}\end{subfigure}\hfill
\begin{subfigure}{0.32\linewidth}\centering
\includegraphics[width=\linewidth]{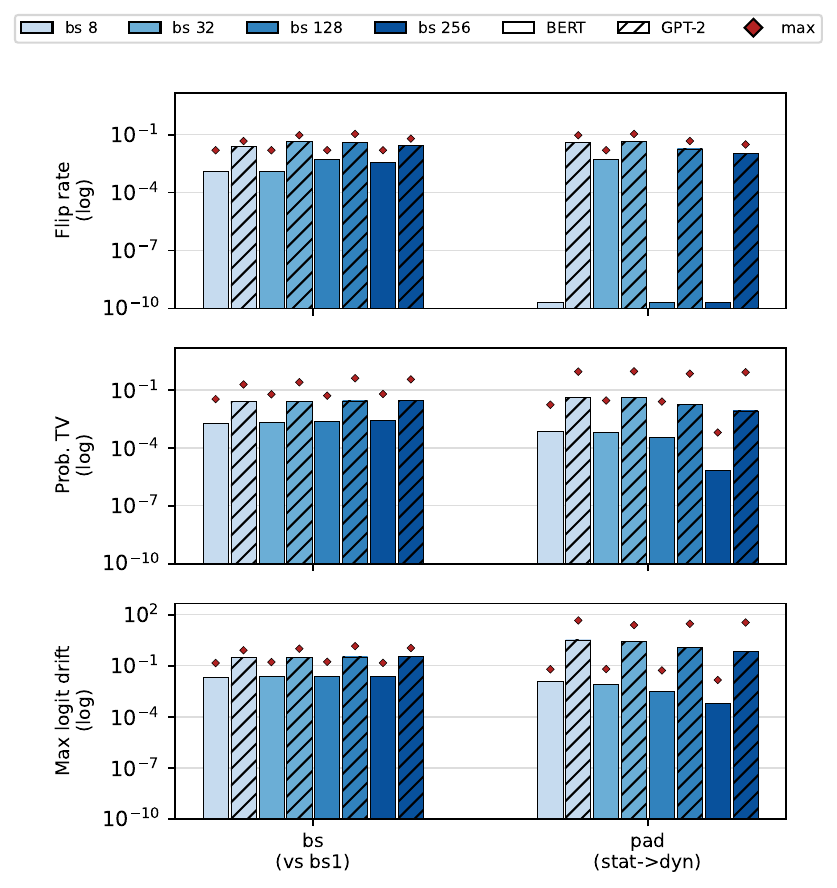}
\caption{8-bit integer, CPU}\end{subfigure}
\caption{Batch size (against the single-example reference) and padded width
(static to dynamic) by deployment; batch size is shown as bar shade. Rows are
label-change rate, probability total variation, and maximum logit change; BERT
solid, fully generative GPT-2 hatched.}
\label{fig:app-knobs}
\end{figure}

\begin{figure}[t]
\centering
\begin{subfigure}{0.32\linewidth}\centering
\includegraphics[width=\linewidth]{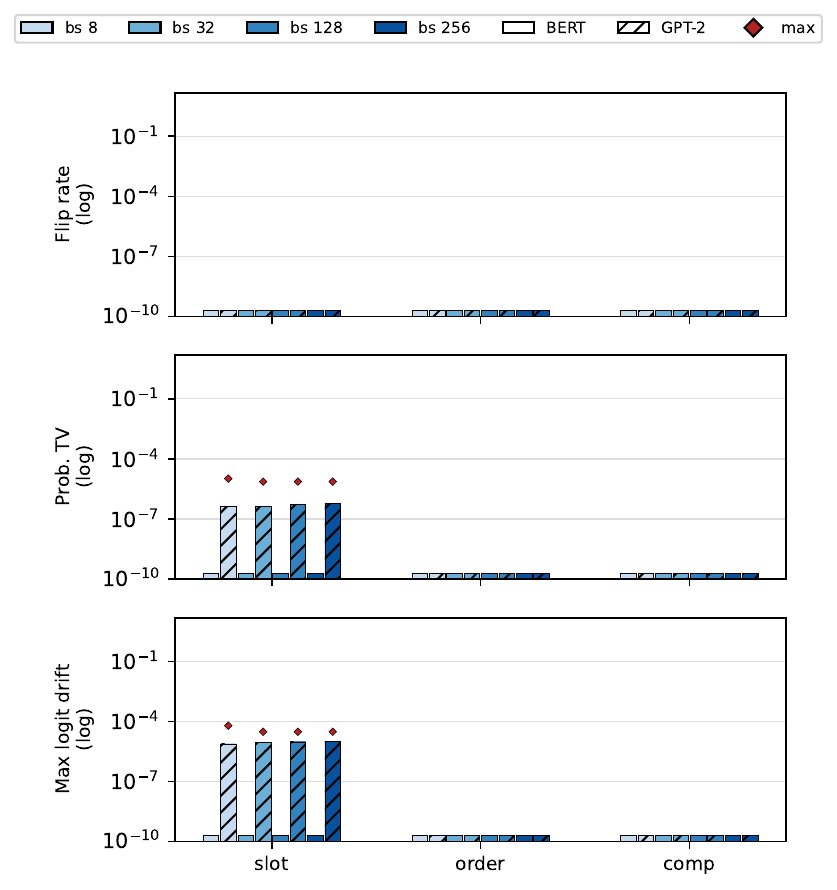}
\caption{full precision, GPU}\end{subfigure}\hfill
\begin{subfigure}{0.32\linewidth}\centering
\includegraphics[width=\linewidth]{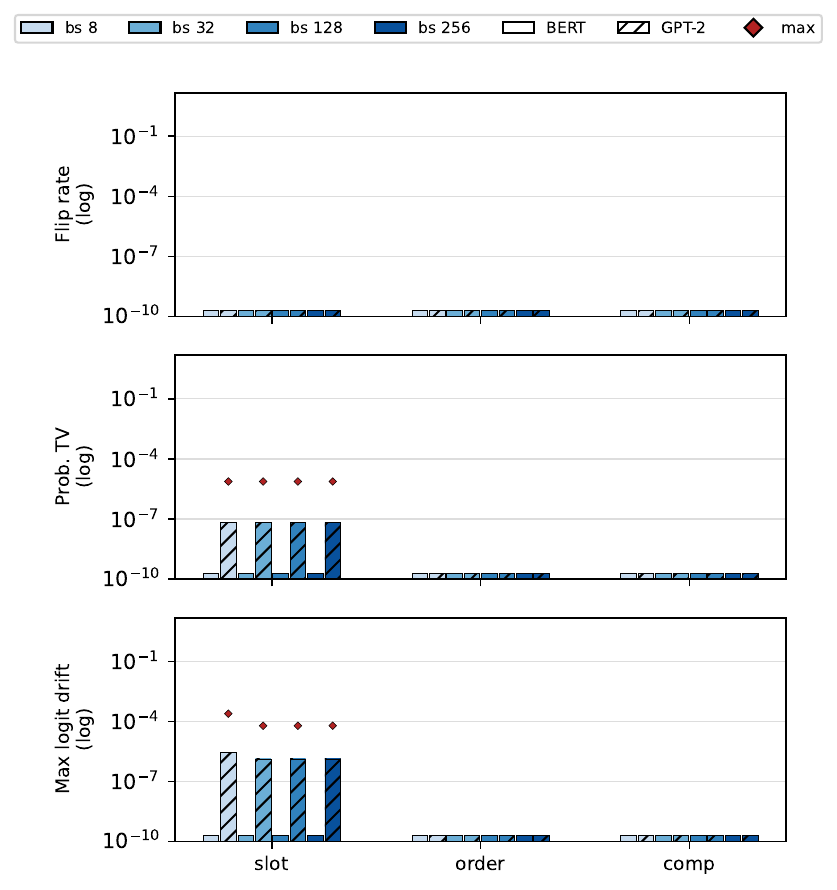}
\caption{bf16, GPU}\end{subfigure}\hfill
\begin{subfigure}{0.32\linewidth}\centering
\includegraphics[width=\linewidth]{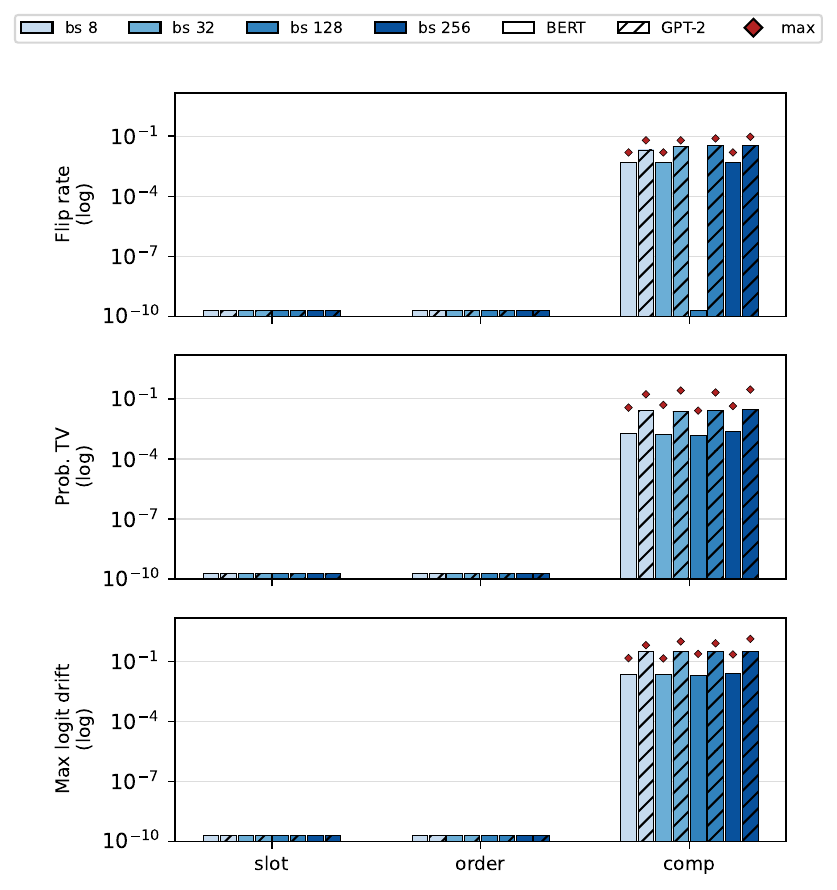}
\caption{8-bit integer, CPU}\end{subfigure}
\caption{Within-batch rearrangements at a fixed batch shape (target slot, order
of the other examples, composition of the other examples) by deployment; batch
size shown as bar shade. Reordering is bitwise at every precision; composition
is bitwise at full precision and bf16 and moves only under dynamic 8-bit
integer quantization; the target slot is bitwise on the integer CPU path but
shows a small nonzero score movement at full precision and bf16 (a
matrix-tiling effect, with no label change).}
\label{fig:app-within}
\end{figure}

\subsection{Serving-context effects by model size}
\label{sec:app-size}

Table~\ref{tab:size} reports the size trends numerically. Figures
\ref{fig:app-knobs-size} and~\ref{fig:app-within-size} give the same
comparisons split by model size (shown as bar shade), restricted to
combinations in which all three sizes pass the accuracy threshold. The score
movement under unchanged labels grows with size while the label-change rate does
not. The plots describe an association with nominal model size, in which
depth, width, head count, and training quality vary together.

\begin{figure}[t]
\centering
\begin{subfigure}{0.32\linewidth}\centering
\includegraphics[width=\linewidth]{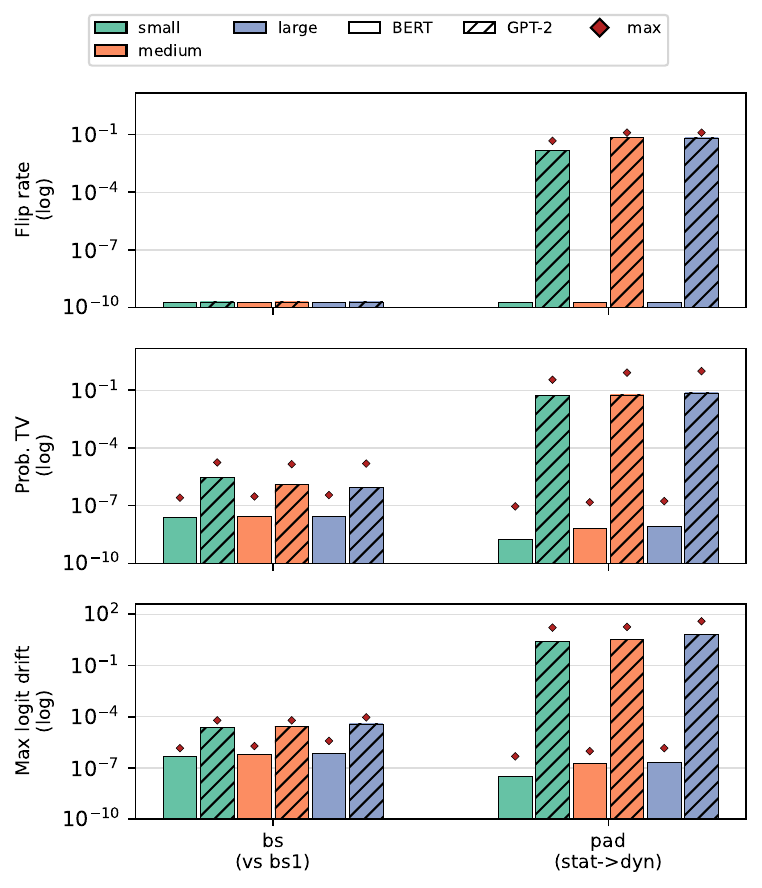}
\caption{full precision, GPU}\end{subfigure}\hfill
\begin{subfigure}{0.32\linewidth}\centering
\includegraphics[width=\linewidth]{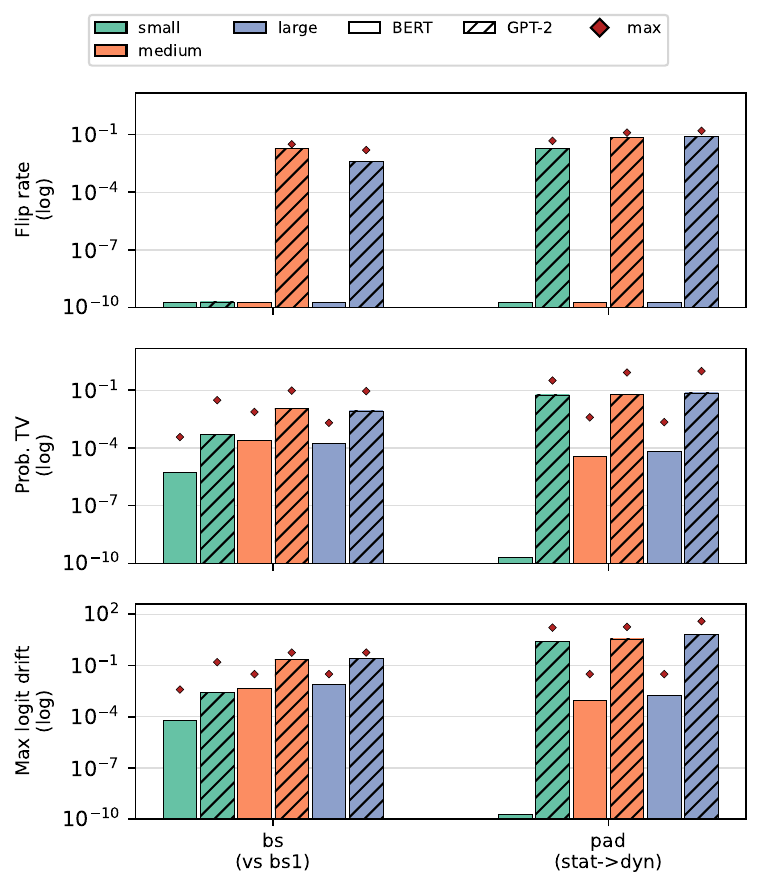}
\caption{bf16, GPU}\end{subfigure}\hfill
\begin{subfigure}{0.32\linewidth}\centering
\includegraphics[width=\linewidth]{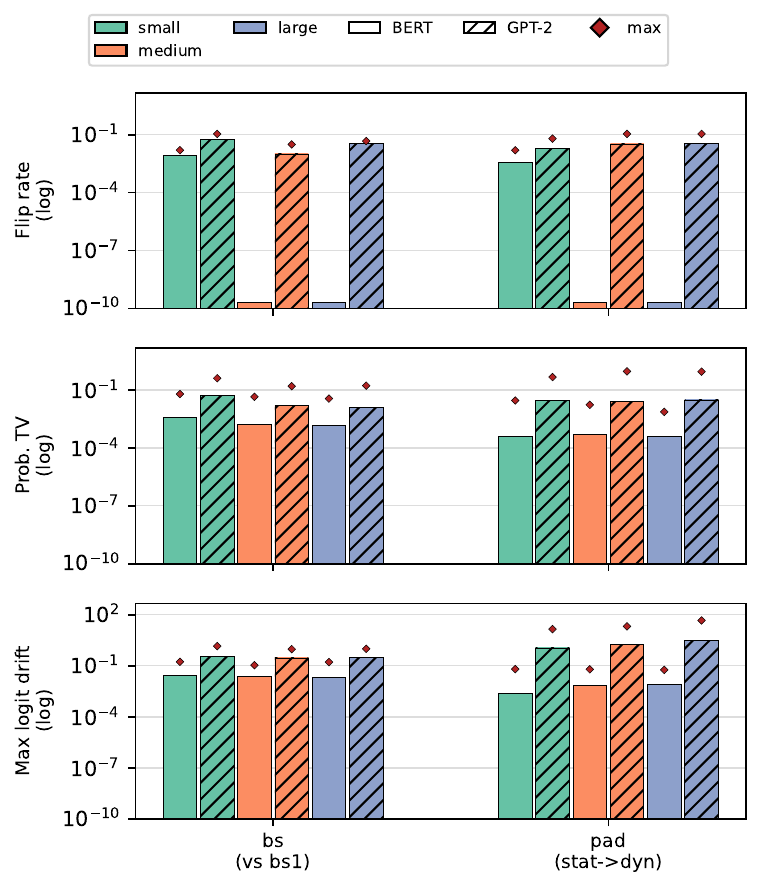}
\caption{8-bit integer, CPU}\end{subfigure}
\caption{Batch size and padded width by deployment and model size (small,
medium, large shown as bar shade), pooled over batch size.}
\label{fig:app-knobs-size}
\end{figure}

\begin{figure}[t]
\centering
\begin{subfigure}{0.32\linewidth}\centering
\includegraphics[width=\linewidth]{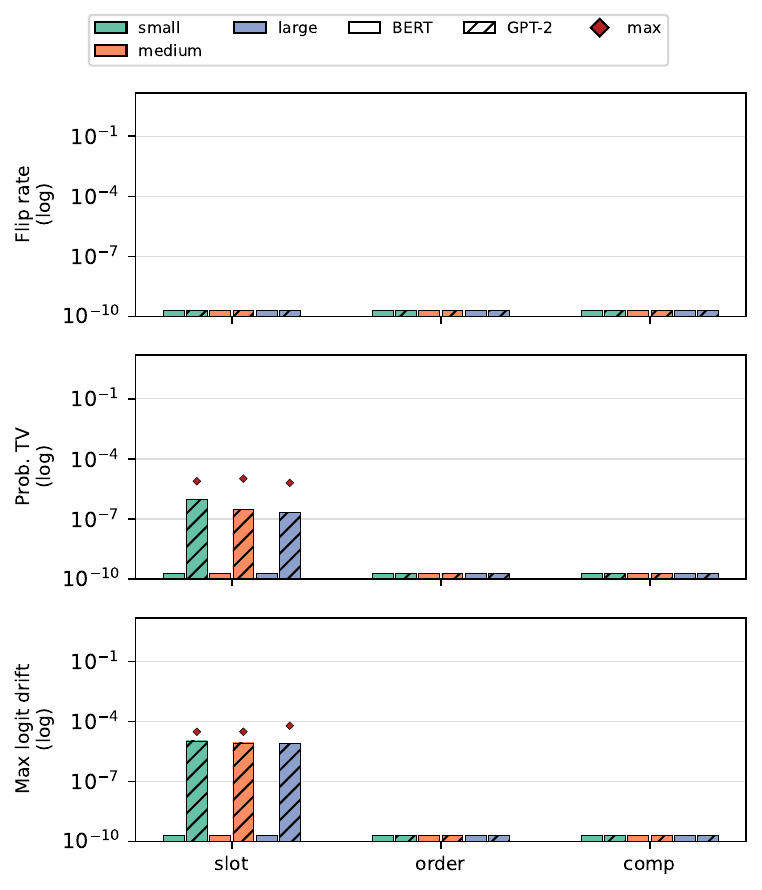}
\caption{full precision, GPU}\end{subfigure}\hfill
\begin{subfigure}{0.32\linewidth}\centering
\includegraphics[width=\linewidth]{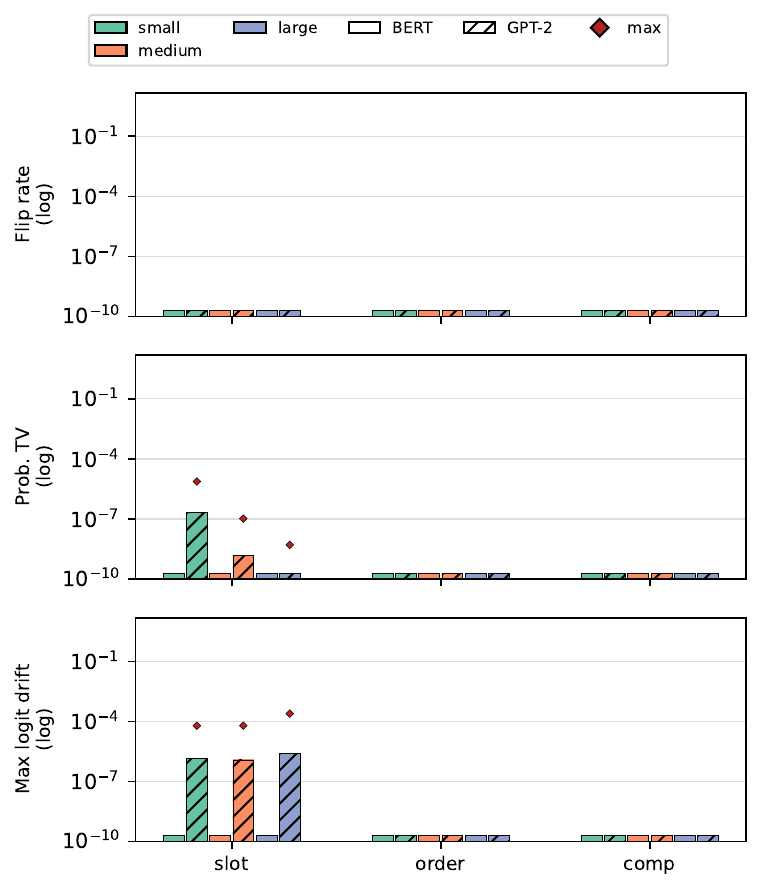}
\caption{bf16, GPU}\end{subfigure}\hfill
\begin{subfigure}{0.32\linewidth}\centering
\includegraphics[width=\linewidth]{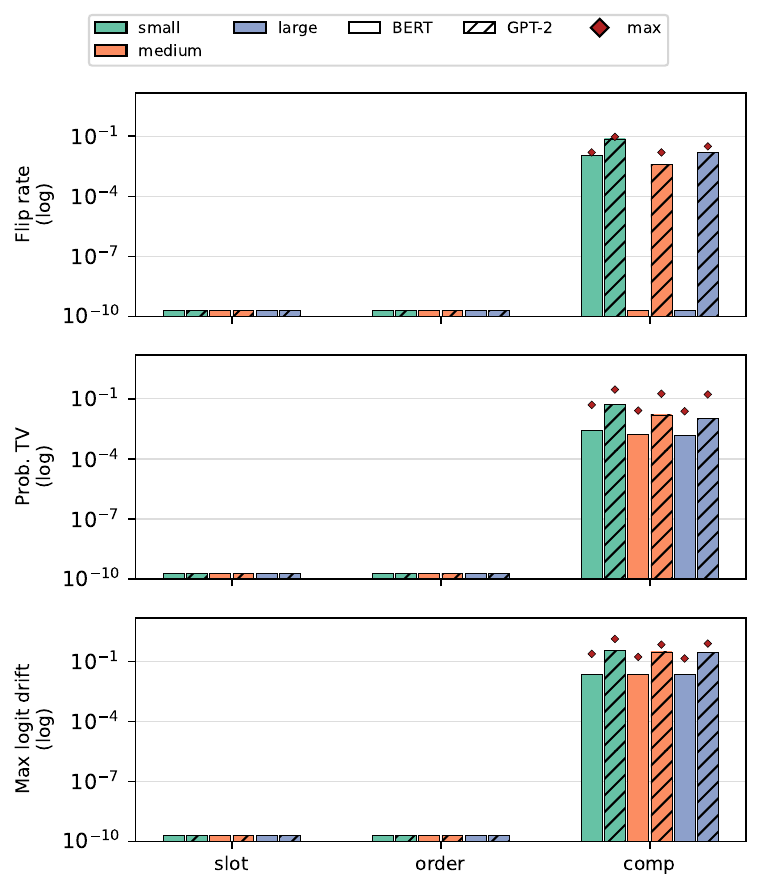}
\caption{8-bit integer, CPU}\end{subfigure}
\caption{Within-batch rearrangements by deployment and model size (small,
medium, large shown as bar shade), pooled over batch size.}
\label{fig:app-within-size}
\end{figure}

\subsection{The padding effect shrinks as the batch grows}
\label{sec:app-padding-batch}

Dynamic padding pads each batch to its longest member, while static padding
pads every batch to the fixed cap. As the batch grows, the longest member
approaches the cap, so the two policies produce increasingly similar padded
widths and the disagreement between them shrinks
(Table~\ref{tab:padding-batch}).

\begin{table}[t]
\centering
\small
\caption{GPT-2 label changes from dynamic versus static padding, by
numerical format and batch size. BERT shows zero changes at every GPU cell
and at most 1 of 192 at a single int8 cell.}
\label{tab:padding-batch}
\begin{tabular}{lcccc}
\toprule
 & \multicolumn{4}{c}{Batch size} \\
\cmidrule(lr){2-5}
Format & 8 & 32 & 128 & 256 \\
\midrule
fp32 (GPU)          & 19/192 (9.90\%)  & 12/192 (6.25\%) & 6/192 (3.13\%) & 3/192 (1.56\%) \\
bf16 (GPU)          & 20/192 (10.42\%) & 16/192 (8.33\%) & 5/192 (2.60\%) & 3/192 (1.56\%) \\
int8, dynamic (CPU) & 16/192 (8.33\%)  & 18/192 (9.38\%) & 7/192 (3.65\%) & 4/192 (2.08\%) \\
int8, static (CPU)  & 15/192 (7.81\%)  & 16/192 (8.33\%) & 9/192 (4.69\%) & 4/192 (2.08\%) \\
\bottomrule
\end{tabular}
\end{table}

\subsection{Precision comparison including fp16}
\label{sec:app-precision-ladder}

fp16 preserved every fp32 label in the tested cells (0 of 192 for BERT and
0 of 192 for GPT-2), while fp32 to bf16 and bf16 to fp16 each moved 2 of
192 GPT-2 labels (1.04\%), so on these checkpoints the roughly one percent label movement is
associated with bf16 rather than with reduced precision as such
(Table~\ref{tab:precision-ladder}). Formats can also select different
kernels, so the result does not establish general fp16 label invariance. The two
formats allocate the same 16 bits differently, since bf16 carries an 8-bit
exponent with a 7-bit mantissa and fp16 a 5-bit exponent with a 10-bit
mantissa, and the maximum logit movement of the bf16 to fp16 step (0.570)
matches that of fp32 to bf16 (0.564 at GPT-2 large), consistent with bf16
being the format that both fp32 and fp16 diverge from. All moved labels
fall in the lowest baseline-margin bin, and BERT remains label-stable on
every step while still showing nonzero score drift.

\begin{table}[t]
\centering
\small
\caption{Precision pairs on the matched six-checkpoint set, 192 comparisons
per family (three model sizes, 64 targets each). Each cell gives label
changes over the denominator and the maximum absolute logit difference.}
\label{tab:precision-ladder}
\begin{tabular}{lcccc}
\toprule
& \multicolumn{2}{c}{BERT} & \multicolumn{2}{c}{GPT-2} \\
\cmidrule(lr){2-3}\cmidrule(lr){4-5}
Precision change & Changes & Max logit drift & Changes & Max logit drift \\
\midrule
fp32 to fp16 & 0/192 & 0.0031 & 0/192 & 0.0638 \\
fp32 to bf16 & 0/192 & 0.0283 & 2/192 & 0.564 \\
bf16 to fp16 & 0/192 & 0.0273 & 2/192 & 0.570 \\
\bottomrule
\end{tabular}
\end{table}

\subsection{Separating the device from the quantization}
\label{sec:app-device-quant}

Changing only the device moved logits by at most $2.86\times10^{-6}$ for
BERT and $1.80\times10^{-3}$ for GPT-2 with zero label changes in 128
comparisons per family, orders of magnitude below the quantization
arithmetic (Table~\ref{tab:device-quant}). Applying 8-bit quantization
with the device held fixed reproduced the movement on either side, with 1
of 192 changes per family when the device was held at the GPU and 1 of 128
GPT-2 changes under dynamic int8 when held at the CPU, every change falling
in the lowest baseline-margin bin. On these subsets, changing the device
alone produced far smaller movement than applying quantization, which
supports quantization as the main source of the whole-migration effect in
Section~\ref{sec:migration}. The controls use different checkpoint subsets
from the main grid and do not decompose that full-grid number. The
conclusion is bounded to the tested pair, one L40S GPU and one AMD EPYC
host under one PyTorch build.

\begin{table}[t]
\centering
\small
\caption{Decomposing the whole-migration change (GPU fp32 to CPU dynamic
int8) into a device factor and a quantization factor. The GPU-held row uses
a six-checkpoint set (192 comparisons per family); the CPU-held rows use a
separate medium-size checkpoint per family on SST-2 and AG News (128 per
family), so the rows are not pooled. The whole-migration rate is measured
on the main grid with different checkpoints.}
\label{tab:device-quant}
\begin{tabular}{lcccc}
\toprule
& \multicolumn{2}{c}{BERT} & \multicolumn{2}{c}{GPT-2} \\
\cmidrule(lr){2-3}\cmidrule(lr){4-5}
Condition & Changes & Max drift & Changes & Max drift \\
\midrule
Device changed, no quantization & 0/128 & $2.86\times10^{-6}$ & 0/128 & $1.80\times10^{-3}$ \\
Device held at GPU, dynamic int8 & 1/192 & 0.055 & 1/192 & 0.276 \\
Device held at CPU, dynamic int8 & 0/128 & 0.054 & 1/128 & 0.585 \\
Device held at CPU, static int8 & 0/128 & 0.030 & 0/128 & 0.301 \\
\midrule
Device changed and dynamic int8 applied & \multicolumn{4}{c}{1.49\% mean label change over 158 pairs} \\
\bottomrule
\end{tabular}
\end{table}

\subsection{Physical mobile devices}
\label{sec:app-mobile}

All phone runs used ONNX Runtime Mobile (ONNX Runtime 1.23.2 for Android)
with fp32, dynamic int8, and static int8 exports of the same BERT SST-2
classifier, evaluated on 64 balanced targets at batch size 8 and width 128
(Table~\ref{tab:phone-controls}). Sessions with the NNAPI accelerator
enabled achieved partial delegation only on the Galaxy S25, and only for
the fp32 and dynamic int8 graphs. The Pixel 8a NNAPI sessions either failed
to create or assigned every node to the CPU provider, so these are partial
placements rather than accelerator-only execution. No supported comparison
changed a label (0 of 64 in every cell). Under dynamic int8, changing the
seven co-batched texts moved all 64 of 64 probability vectors, with maximum
total variation 0.0048 on the phone CPUs and 0.0057 under the partial
Galaxy NNAPI placement, where it also produced one threshold-action change
and moved four of 64 targets by one rank. This matches the server-side
dynamic int8 composition result, and the static int8 export was bitwise
identical under the same composition change on both phone CPUs. The target
slot and the order of the other examples were bitwise identical for every
format on every supported placement.

\begin{table}[t]
\centering
\small
\caption{Serving-context factors on physical phones. Entries give the
maximum total variation of the target probability vector, or ``bitwise''
when all 64 of 64 targets were bitwise identical, for CPU execution. Cells
marked n.m.\ were not measured because the int8 graphs were exported at a
fixed batch size of 8.}
\label{tab:phone-controls}
\begin{tabular}{llcccc}
\toprule
Device & Format & Batch size & Composition & Target slot & Order \\
\midrule
Galaxy S25 & fp32         & bitwise & bitwise & bitwise & bitwise \\
Galaxy S25 & dynamic int8 & n.m.    & 0.0048  & bitwise & bitwise \\
Galaxy S25 & static int8  & n.m.    & bitwise & bitwise & bitwise \\
Pixel 8a   & fp32         & bitwise & bitwise & bitwise & bitwise \\
Pixel 8a   & dynamic int8 & n.m.    & 0.0048  & bitwise & bitwise \\
Pixel 8a   & static int8  & n.m.    & bitwise & bitwise & bitwise \\
\bottomrule
\end{tabular}
\end{table}

\subsection{Downstream decisions move more than labels}
\label{sec:app-decisions}

Table~\ref{tab:app-decision-rates} reports, for three serving
interventions, the fraction of evaluation-example pairs whose predicted
label, fitted per-class threshold action, abstention decision at 90 percent
target coverage, 90 percent prediction-set membership, or top 1 percent
ranking membership changed. Every policy was fitted on reference-context
outputs only, then frozen before any evaluation metric was produced, so no
threshold was refitted to a serving context. In the tested configurations
the threshold, abstention, prediction-set, and ranking quantities change
more often than the predicted label under padding and precision changes,
and the batch-size comparison changes zero labels among 799{,}620 pairs
while still moving threshold actions, prediction-set membership, and
ranking membership at small nonzero rates. These are rates over a balanced
experimental grid, not production traffic prevalence.

\begin{table}[t]
\centering
\small
\caption{Fraction of evaluation-example pairs (percent) with a changed
decision under each intervention. Abstention and prediction sets use 90
percent target coverage. Ranking churn is the fraction of the top 1 percent
set whose membership differs, $|S \triangle S'| / 2k$ for selected sets of
size $k$ per class, not a fraction of all examples.}
\label{tab:app-decision-rates}
\setlength{\tabcolsep}{4pt}
\begin{tabular}{lrrrrrr}
\toprule
Intervention & Pairs & Label & Threshold & Abstention & Pred.\ set & Top 1\% rank \\
\midrule
Padding policy & 218{,}439 & 10.224 & 18.546 & 9.161 & 15.468 & 32.010 \\
Precision (fp32 to bf16) & 799{,}620 & 0.991 & 3.077 & 2.141 & 1.932 & 12.655 \\
Batch size (32 to 8) & 799{,}620 & 0.000 & 0.0004 & 0.000 & 0.0003 & 0.0042 \\
\bottomrule
\end{tabular}
\end{table}

\subsection{Recomputing low-margin predictions in a canonical context}
\label{sec:app-replay}

A deployable mitigation recomputes the least confident fast-path
predictions in a declared canonical context and releases the canonical
result for those predictions. The fast path here is bf16 execution at batch
size 32, and the canonical context is fp32 execution at batch size one with
static width, for the GPT-2 medium SST-2 checkpoint. Each policy ranks
predictions by a fast-path risk score, the top-two probability margin for
the label policy and the distance of the positive-class probability from
the declared 0.8 action threshold for the fixed-threshold policy. Cutoffs
were fitted on a calibration split of 547 targets that was fixed before
evaluation on the 1{,}274 held-out targets; because the calibration split
contained fewer disagreements than the predeclared minimum, the fitted
operating points are underpowered. The label-margin policy recomputed
22.45\% of requests and detected all 44 observed label disagreements, while
the fixed-threshold policy recomputed 4.00\% and detected 18 of 21
threshold-action disagreements (Table~\ref{tab:app-replay}). The two
policies cover different decisions, and neither covers complete
probability vectors.

\begin{table}[t]
\centering
\small
\caption{Held-out replay results for GPT-2 medium on SST-2 at fast batch
size 32 (1{,}274 held-out targets).}
\label{tab:app-replay}
\begin{tabular}{lrr}
\toprule
Policy & Recompute rate & Disagreements detected \\
\midrule
Label margin & 22.45\% (286 of 1{,}274) & 44 of 44 \\
Fixed threshold & 4.00\% (51 of 1{,}274) & 18 of 21 \\
\bottomrule
\end{tabular}
\end{table}

\subsection{Output-position faults that mimic padding instability}
\label{sec:app-readout}

Some apparent padding instability is an implementation fault in the
classification head rather than an effect of the position rule. Under valid
fixed-shape repacking of batch neighbors, every rule that locates the
correct semantic token or pools over exactly the packed set of real tokens
changed 0 of 64 target labels for both tested checkpoints, and the BERT
semantic-token gathers were bitwise identical for 64 of 64 targets.
Rules that instead read a fixed physical column or include padding
positions in the pooled set changed labels at fixed batch shape, from 3 of
64 for the BERT fixed-last-column rule up to 64 of 64 for the GPT-2
fixed-first-column rule under end-to-end left versus right placement,
with the GPT-2 pad-inclusive mean and sum each changing 28 of 64. These
faults are distinct from the structural mechanism in which real-token
positions depend on the padded width, and the two must be fixed separately,
since correcting the positions does not correct an unsafe output position and
correcting the output position does not remove width-dependent positions. The
practical rule is to derive the output position from the attention mask and
to exclude padding positions from pooling.

\subsection{Flip-rate distribution across the main grid}

This table ranks the serving-context comparisons by how often they change the
predicted label, from the input-side padding effect down to the within-batch
rearrangements that change no observed labels.

\begin{table}[h]
\centering
\tiny
\caption{Serving-context instability among the 158 from-scratch checkpoints
that exceed 1.15 times chance validation accuracy. Means average designed pairs rather than production
traffic. P90 is the 90th percentile of pair-level flip rates.}
\begin{tabular}{lrrrrrr}
\toprule
Comparison & Pairs & Mean flip & 95\% CI & Median & P90 & Maximum \\
\midrule
Padding & 158 & 0.10297 & [0.07862, 0.12987] & 0 & 0.40835 & 0.671 \\
Precision & 790 & 0.01619 & [0.00968, 0.02429] & 0.00250 & 0.03032 & 0.404 \\
Whole migration (GPU fp32 to CPU int8) & 158 & 0.01492 & [0.00940, 0.02229] & 0.00439 & 0.03114 & 0.386 \\
Batch size & 3,792 & 0.00019 & [0.00010, 0.00029] & 0 & 0 & 0.052 \\
Global order & 948 & $\approx0$ & [0, 0.00001] & 0 & 0 & 0.001 \\
Target slot & 474 & \multicolumn{5}{c}{zero observed} \\
Fixed-shape composition & 948 & \multicolumn{5}{c}{zero observed} \\
\bottomrule
\end{tabular}
\end{table}

\clearpage

\subsection{Scope of the order and composition comparisons}

The order comparison permutes the full test list before fixed-width batching.
It preserves target IDs and tensor width but can change peers and target slot,
so we do not interpret it as a pure within-batch permutation effect. The
composition comparison holds the padded width fixed and changes which other
examples share the batch. It is the clean test of whether their content
reaches the target. A separate random-peer width
experiment allows peer lengths to change dynamic width and is therefore a
joint peer/shape intervention.

\section{Implementation Details}
\label{app:implementation}

\paragraph{Primary models and tokenization.}
Small, medium, and large models use hidden sizes 64, 384, and 768; depths
1, 6, and 12 layers; 1, 6, and 12 attention heads; and feed-forward widths
256, 1,536, and 3,072, respectively, with maximum sequence length 512. BERT families use the
\texttt{bert-base-uncased} WordPiece vocabulary and right padding. GPT-2
families use the \texttt{gpt2} byte-pair vocabulary, reuse EOS as PAD, and are
manually left padded. The generative formulations add one generation marker and
one synthetic token per class.

\paragraph{Objectives and scores.}
The direct BERT model uses sequence-classification cross-entropy. The
masked-label model applies standard masked-language-model cross-entropy,
randomly masks 15\% of eligible tokens, and always masks the synthetic label
positions. The next-label GPT-2 model trains on
\texttt{text [label\_gen] [label\_y]} and classifies from the next-token
scores after the generation marker. The class-conditional GPT-2 model trains
on \texttt{[label\_gen] [label\_y] text} and scores each candidate class by
summing non-padding token log probabilities. At inference, every formulation's
class scores are converted to fp32 and softmax-normalized over only the
candidate classes. The reported generative probabilities are therefore
class-normalized verbalizer or sequence-score probabilities, not
full-vocabulary probabilities.

\paragraph{Primary optimization.}
We use AdamW with learning rate $2\times10^{-5}$, zero warm-up, linear decay,
batch size 32, at most 60 epochs, and early-stopping patience 10. Training and
checkpoint selection use validation weighted F1. The best of the retained
checkpoints is evaluated. SST-2, SST-5, and Emotion use supplied validation
splits. AG News and IMDb use a stratified 20\% training carve-out with split
seed 42. Runs use fp32 and seeds 40 to 42 under the strict deterministic
profile.

\paragraph{Modern pretrained models.}
ModernBERT-base and Qwen3-0.6B use LoRA rank 16, alpha 32, dropout 0.05, and
all linear modules. We optimize trainable parameters with AdamW at
$2\times10^{-4}$ for at most 10 epochs, with zero warm-up, linear decay, and
patience 3. Adapters are merged before inference. ModernBERT uses
sequence-classification cross-entropy. Qwen uses left-padded raw text and
class cross-entropy over final-position logits for single-token digit labels.

\section{Extended Related Work}
\label{app:related}

This appendix expands the related work summarized in the main text. It groups prior work into six areas and, for each, states what that work establishes and how it differs from the question studied here, namely how the prediction of a single fixed trained classifier changes when only the inference-time serving context varies while the model checkpoint, the input text, and the tokenizer are held fixed.

\paragraph{Discriminative, generative, and pseudo-generative text classifiers.}
A text classifier can be built in several ways that differ in what probability the model represents. A discriminative classifier estimates the conditional probability of the label given the input text directly, for example logistic regression or a neural encoder with a classification head. A generative classifier instead models the distribution of the input, either the joint distribution of text and label or the class-conditional distribution of the text, and turns this into a label decision through Bayes' rule. The comparison between these two families has a long history, beginning with analyses of logistic regression, discriminant analysis, and naive Bayes, and of the conditions under which one family is more sample efficient or more accurate than the other \citep{efron1975efficiency,ng2001discriminative,xue2008comment}. Naive Bayes text classification, in which the class-conditional distribution of words is modeled under an independence assumption, is an early and widely used generative approach for text \citep{sahami1998bayesian}. Neural formulations extended the comparison. This work includes recurrent generative and discriminative text classifiers \citep{yogatama2017generative}, transformer-based discriminative and generative models for sentence-level classification \citep{rezaee2021discriminative,nogueira2019passage}, and a theoretical account of when each family generalizes better \citep{zheng2023revisiting}. A third formulation has become common with pretrained language models. A masked-language-model classifier scores a label by how well a language model, trained to predict tokens that have been hidden, reconstructs a label-conditioned version of the input. Because the score is a product of conditional token probabilities rather than a properly normalized joint density over the whole sequence, it approximates a likelihood through pseudo-likelihood \citep{wang2019bert}. We call this formulation pseudo-generative to separate it from a classifier that models a normalized joint distribution. Recent work uses autoregressive language models directly for classification \citep{sun2023text} and compares encoder, autoregressive, masked-language-model, and discrete-diffusion classifiers within one transformer-era study \citep{kasa2025generative}. Related work examines other design choices in text classification, such as the treatment of ordinal label structure \citep{kasa2024ordinality}. Further studies characterize properties that separate the families. Fully generative classifiers are not automatically resistant to shortcut features that correlate with the label in training but not in deployment \citep{li2025shortcut,kasa2026shortcut}, and modeling the input distribution can raise the risk that an attacker infers whether a specific example was in the training set \citep{makroo2026hidden}. These studies compare training objectives, model formulations, sample efficiency, robustness, calibration, and privacy across classifier families. None of them holds one fixed trained checkpoint and one fixed input text and varies only the serving context, such as padding, batch composition, numerical precision, quantization, device, or inference backend, to measure whether the prediction of that fixed model changes.

\paragraph{Non-determinism and non-invariance in language-model inference.}
A separate line of work documents that the output of a language model can change when the way the computation is executed changes, even though the model parameters, the prompt, the decoding parameters, and any sampling seed are held fixed. The root numerical cause is that floating-point addition is not associative. Adding the same set of numbers in a different order can yield a slightly different rounded result, because each intermediate sum is rounded to a finite number of bits \citep{goldberg1991floating}. This matters in practice because a serving system computes sums, such as the dot products inside a matrix multiplication or a reduction across the sequence, by splitting them across parallel hardware threads, and it selects which low-level routine (kernel) to run and how to partition the reduction as a function of the shape of the input batch. When the batch shape changes, for instance because more or fewer requests are processed together, the partitioning and therefore the accumulation order can change, and the rounded result moves \citep{yuan2025understanding,he2025defeating,zhao2026corun}. The choice of inference engine and library version is a further source of divergence for the same model and inputs \citep{pape2026silent}. At the level of hosted services, repeated requests under nominally deterministic settings change the generated output and the benchmark accuracy computed from it, without identifying which serving factor moved \citep{atil2024non}. That measurement is generation-mediated, because the task answer is parsed out of generated text. Three responses to this have appeared. First, kernels can be written so that the reduction order does not depend on the batch shape, which restores identical results across batch sizes at a fixed execution configuration \citep{he2025defeating}. Second, requests can be padded so that each is processed at a fixed tensor shape, keeping a request's arithmetic independent of which other requests share its batch \citep{zhao2026corun}. Third, the default kernels can be kept and every decoded token verified against a deterministic recomputation, with disagreements repaired \citep{gond2026llm}. The margin-based recomputation we measure in Section~\ref{sec:mitigations} is the classification counterpart of this response, applied once per input rather than once per decoded token. The invariance that this padding-for-determinism work restores concerns a request that occupies a different row of the batch while the tensor shape is held fixed. It is distinct from the change studied in this paper, in which a real token occupies a different absolute position within its own sequence because the padded width itself changed. Related measurement work shows that reporting only aggregate, dataset-level metrics hides the per-example movement these effects cause \citep{fang2026dataset}, and that the token probabilities of a language model reveal non-determinism that a discrete top-one output conceals \citep{fu2026beyond}. All of this work studies language-model generation, autoregressive or diffusion-based, in which the output is a decoded sequence and a single perturbed step can propagate through the remaining decoding. None of it isolates the single forward pass of a fixed text classifier, holding the checkpoint and the input text fixed, to determine whether and by how much these execution-level mechanisms change a classification decision rather than a generated sequence.

\paragraph{Prediction instability, churn, and underspecification.}
A different body of work studies how the predictions of a classifier vary across models rather than across executions of one model. Fine-tuning a pretrained model is itself unstable. The final model depends on the random initialization of the added parameters and on the order in which training examples are presented, so the same recipe run twice can yield different models \citep{dodge2020finetuning}. Prediction churn names the fraction of examples whose predicted label differs between two successively trained models, and procedures exist to reduce it during retraining \citep{milanifard2016launch}. Underspecification describes the situation in which a training pipeline admits many distinct solutions that reach nearly the same held-out accuracy yet behave differently on axes that matter in deployment, so aggregate accuracy does not single out one model \citep{damour2022underspecification}. Predictive multiplicity makes the same point at the level of individual inputs. Several models with near-equal accuracy can assign conflicting predictions to the same example, and this multiplicity can be quantified over the set of near-optimal models \citep{marx2019predictive,hsu2022rashomon}. In all of these settings the source of variation is a change of model, whether across independent training runs or across the set of models a pipeline could have produced. This paper instead holds a single trained model fixed and varies only the inference-time serving context, so the disagreement it measures is not attributable to a different choice of parameters.

\paragraph{Imprecise probability, credal classification, and calibration.}
Some formal frameworks represent uncertainty with more than a single probability distribution. Imprecise probability describes partial or unresolved knowledge using sets of distributions together with lower and upper probabilities, rather than one sharp distribution \citep{walley1991statistical}. A credal classifier applies this idea to prediction and may return a set of admissible labels, rather than forcing a single label, when the available evidence does not determine one. The naive credal classifier and its later refinement are the canonical examples for categorical data \citep{zaffalon2002naive,corani2008reliable}. The set of predictions studied in this paper is related in outward form to a set-valued prediction but different in origin and in intended use. It is the finite collection of outputs that one fixed model actually produces across the serving configurations that are tested, with no assumed convexity of a probability set, no generating prior, and no belief-function semantics. Taking the union of the predicted labels across configurations is one possible set-valued reading of that collection, not a decision rule imported from the credal literature. Calibration is a further, separate notion. It measures whether the confidence a model reports matches the empirical frequency with which such predictions are correct, and modern neural networks are often miscalibrated toward overconfidence \citep{guo2017calibration}. The distinction between uncertainty that is irreducible given the data and uncertainty that reflects limited knowledge and could in principle be reduced has been formalized as aleatoric versus epistemic uncertainty \citep{hullermeier2021aleatoric}. This paper does not estimate calibration error and does not decompose uncertainty into these components. It measures how much a fixed model's output depends on hidden serving configuration.

\paragraph{Numerical and positional mechanisms.}
Two implementation choices are central to how serving context can change a prediction. The first is the numerical format used to store values and to carry out arithmetic. Mixed-precision execution keeps some values in a reduced-precision floating-point format, which uses fewer mantissa bits than the standard single-precision format and therefore represents each number more coarsely, while accumulating selected sums in higher precision to limit error \citep{micikevicius2018mixed}. Quantization is a distinct technique that maps floating-point values to low-bit integers through a scale factor and then performs integer arithmetic, for example 8-bit integer matrix multiplication for transformer models \citep{dettmers2022llmint8}. The term quantization is reserved here for such integer methods and is not used for reduced-precision floating point. Both change how the same represented input is computed numerically, and this paper measures their effect on the individual predictions of a fixed model. The second choice is how position information is supplied to the model, which matters because padding changes the absolute position index assigned to each real token. Learned absolute position encoding, used by the original bidirectional encoder and by the autoregressive decoder considered here, assigns every absolute position index its own learned vector that is added to the token embedding, so the representation of a token depends on its absolute index \citep{devlin2019bert,radford2019gpt2}. Relative schemes instead encode the offset between a query position and a key position rather than the two absolute indices \citep{shaw2018relative}. Rotary position embedding rotates the query and key vectors by an angle proportional to the absolute position, so that the attention score depends only on the relative offset and a shift common to query and key cancels \citep{su2024roformer}. Attention with linear biases adds a penalty to attention scores that grows with the distance between positions \citep{press2022alibi}. Models trained with no explicit positional encoding can still recover position information through the causal attention mask \citep{haviv2022nope}. The choice of encoding scheme also affects how well a model generalizes to sequence lengths not seen during training \citep{kazemnejad2023impact,wang2024length}. Current encoder and decoder models widely adopt rotary position embedding \citep{warner2024modernbert,qwenteam2025qwen3}. A related training-time intervention adds random padding during training so that real tokens are seen at many absolute positions, which improves the robustness of learned position embeddings \citep{tao2023randompadding}. Together these mechanisms explain why a change in padded width can change the represented input for a model whose token representations depend on absolute position, while leaving it unchanged for a scheme that depends only on relative offsets. Because the architecture name alone does not settle the question, the effect is measured directly rather than assumed.

\paragraph{Behavioral and production-readiness testing.}
A final line of work tests machine learning systems beyond a single aggregate accuracy number. Behavioral testing constructs targeted suites that test specific capabilities and failure modes of a natural language processing model, exposing errors that overall accuracy hides \citep{ribeiro2020checklist}. Production-readiness and reliability frameworks catalog the operational hazards of deployed systems, including hidden dependencies on data and configuration, and propose checklists that span data, model, and infrastructure \citep{sculley2015debt,breck2017mltest}. The checks introduced in this paper are complementary to both. For a fixed model and a fixed input text, they compare the predicted labels and the underlying scores across serving configurations. They require every comparison to pass a task-quality threshold, so that a configuration in which the model has collapsed to a constant prediction is excluded rather than counted as agreement. They also include a deliberately vulnerable configuration whose sole purpose is to confirm that each measurement is able to detect the effect it is designed to test for.


\section{Theoretical Results and Proofs}\label{app:theory}

Throughout this appendix, $z(x,c) \in \mathbb{R}^{K}$ denotes the logit
vector produced for input $x$ under serving context $c$, where a serving
context specifies the hidden execution conditions (batch shape, padded
width, numerical precision, quantizer state, and kernel configuration)
while the checkpoint, the text, and the tokenizer are held fixed. Class
probabilities are obtained by applying softmax to the displayed logits,
$p(y \mid x,c) = \operatorname{softmax}(z(x,c))_y$, and ties are broken by
one fixed deterministic rule. The total variation distance between the
probability vectors of two contexts is
$\operatorname{TV}(x,c,c') = \tfrac{1}{2}\sum_y |p(y \mid x,c) - p(y \mid x,c')|$.
The top-two margin under a reference context $c_0$ is the gap between the
winning score and the strongest competitor, measured either on logits,
$g(x) = z_w(x,c_0) - \max_{j \neq w} z_j(x,c_0)$, or on probabilities,
$m(x) = p_w(x,c_0) - \max_{j \neq w} p_j(x,c_0)$, where $w$ is the unique
reference winner. The results below formalize the mechanisms studied in the
empirical sections: the structural position-assignment mechanism, the margin
conditions for label stability, the quantization scale-policy contrast, the
accumulation behavior of fully generative scoring and greedy decoding, the
propagation of numerical discrepancies through depth, and the conditional
feasibility of jointly attaining accuracy, speed, and serving-context
invariance.

\subsection{Position assignment under changes of padded width}

\begin{cprop}{Sufficient common-shift condition for width invariance}{common-shift}
Let the real tokens of a fixed input occupy logical positions $i$ and $j$,
and let adding $s$ leading pad columns move their physical tensor columns
to $i+s$ and $j+s$. Let $\Phi(i,j)$ denote every positional contribution to
the real-token computation. Suppose all of the following hold:
\begin{enumerate}
\item Every positional contribution satisfies $\Phi(i+s,\, j+s) = \Phi(i,j)$.
\item Padding is excluded from attention and pooling, so the visibility
relation among real tokens is unchanged.
\item The class scores are read from the same semantic token or the same masked set of
real tokens.
\item Cache offsets, if present, use one consistent logical coordinate
system.
\item No other operation depends on the absolute physical tensor column.
\end{enumerate}
Then a common shift of all real-token columns leaves the ideal represented
real-token computation unchanged in exact arithmetic.
\end{cprop}

\begin{proof}
The base real-token contents are unchanged by the shift, since padding
inserts columns without altering any real token embedding. Condition 1
preserves every positional term entering the computation, so the
first-layer inputs restricted to real tokens are identical in the two
executions. Condition 2 preserves the dependency relation and the
aggregation domain of every attention and pooling operation over real
tokens, so each layer combines the same set of real-token states with the
same positional terms. Conditions 3 and 4 preserve the semantic output position
and any cached positional relations. Condition 5 excludes any
unmodeled path through which the absolute physical column could enter the
computation. Proceeding by induction over layers, the real-token states
after each layer are identical in exact arithmetic, and therefore so are the
class scores read from them.
\end{proof}

This proposition is a sufficient positional-mechanism condition, not a
complete serving-invariance theorem. A wider tensor can still select
different kernels, tilings, or reduction trees and create numerical drift
even when the represented input is preserved. The condition explains the
architecture control table. Models whose positional terms depend only
on token distance (rotary position embeddings under left padding, and
encoders whose real-token positions do not move under right padding) showed
zero observed label changes out of 29{,}463 comparisons each. A
natively trained model with no explicit position term showed zero label
changes out of 256 with only fp32-scale numerical differences.

\begin{cprop}{Learned absolute positions change the represented input}{absolute-position}
Let the initial representation of real token $x_i$ under a width-derived
left-padding rule be $h_i^{(0)} = e(x_i) + p(\pi_i)$, where positions are
taken from the physical tensor column, so that adding $s$ leading pad
columns yields the shifted representation
$\widetilde{h}_i^{(0)} = e(x_i) + p(\pi_i + s)$. If
\[
p(\pi_i + s) \neq p(\pi_i)
\]
for any real token, then the two executions receive different represented
inputs. Consequently the same-input premise of any numerical
batch-invariance guarantee is violated, and such a guarantee does not imply
equal outputs.
\end{cprop}

\begin{proof}
The difference between $h_i^{(0)}$ and $\widetilde{h}_i^{(0)}$ is present
before the first transformer layer. The same implemented function is then
evaluated on two different represented inputs. The outputs may
coincidentally agree, but a numerical guarantee whose premise is an
identical represented input does not apply, because that premise fails.
Correct masking does not remove this mechanism, since the changed vectors
belong to real tokens rather than pads, and a semantic output position still
consumes the changed hidden states.
\end{proof}

This converse explains the other half of the architecture table. The
left-padded family with learned absolute positions showed 35{,}591 label
changes over 176{,}778 comparisons (mean 20.1\%; 10.3\% when all comparisons from the 158 pairs that
passed the accuracy threshold are pooled together), while the right-padded controls whose real-token
positions do not move showed zero observed label changes over the same
comparison count.

\begin{cprop}{Instantiation of the common-shift condition across position-encoding schemes}{schemes}
Let all real-token tensor columns shift by a common offset $s$, and let
$\Phi(i,j)$ denote the positional contribution to the interaction between
positions $i$ and $j$. The common-shift condition
$\Phi(i+s, j+s) = \Phi(i,j)$ holds for the following schemes.
(i) Positions assigned by order among real tokens: the assigned indices
$0, \dots, L-1$ do not change at all, so every positional term is
identical.
(ii) Rotary embeddings: with rotation operator $R$, the query-key
positional term satisfies
$\langle R(i)q, R(j)k \rangle = \langle q, R(j-i)k \rangle$, a function of
$j - i$ only.
(iii) ALiBi: the attention bias is $B_h(i,j) = -m_h (i - j)$, a function
of the distance only.
(iv) T5 bucketed relative bias: $B_h(i,j) = b_h[\operatorname{bucket}(j-i)]$.
(v) Shaw relative representations: relation vectors are selected by
$r_{\operatorname{clip}(j-i,\,k)}$.
(vi) Transformer-XL relative attention: the relative-position terms depend
on $i - j$, provided cache indices use one coherent coordinate system.
(vii) A model with no positional term satisfies the condition vacuously.
The condition fails for learned absolute embeddings indexed by tensor
column, since $p(\pi_i + s) \neq p(\pi_i)$ in general, and it is not
satisfied in general by sinusoidal absolute encodings. Although
$PE(i+s) = A_s\,PE(i)$ for a block-rotation $A_s$ within each frequency
pair, the encoding is added to token content before arbitrary learned
projections, and no general end-to-end cancellation identity exists.
\end{cprop}

\begin{proof}
For (i), the position assignment is computed from the mask, so adding pad
columns leaves every assigned index unchanged, and each positional term is
evaluated at identical arguments. For (ii) through (vi), substitute
$i+s$ and $j+s$ into the stated positional term. In each case the term
depends on its arguments only through the difference
$(j+s) - (i+s) = j - i$ (after clipping or bucketing, which are functions
of the difference), so the term is unchanged. For (vii) there is no
positional term to change. For learned absolute embeddings, the initial
representation of a real token becomes $e(x_i) + p(\pi_i + s)$, which
differs from $e(x_i) + p(\pi_i)$ whenever the embedding table distinguishes
the two rows, establishing the failure by Proposition~\ref{propx:absolute-position}.
For sinusoidal encodings, the block-rotation identity shows that a common
shift acts linearly on the encoding. However, the encoding enters the model
additively in content space and passes through learned projections that
need not commute with $A_s$; absent such structure, the represented input
changes, so the condition is not satisfied in general.
\end{proof}

Among these schemes, we directly
measure learned absolute positions (affected), order positions, rotary
embeddings (Qwen3 and ModernBERT), and a native model trained without any
positional term (all unaffected in the measured cells). The ALiBi, T5,
Shaw, Transformer-XL, and sinusoidal statements are mechanism-level
consequences of the published formulas and are literature-supported
boundaries, not measurements from our experiments. End-to-end invariance for any
implementation additionally requires the masking, output-position, cache, and
no-other-path conditions of Proposition~\ref{propx:common-shift}.

\subsection{Margin certificates for label stability}

\begin{cprop}{Sufficient logit-margin condition}{logit-margin}
Let $w$ be the unique winning class under the reference context $c_0$ and
let $g(x) = z_w(x,c_0) - \max_{j \neq w} z_j(x,c_0)$ be the reference
top-two logit margin. Suppose that for every allowed context $c$,
\[
\left\| z(x,c) - z(x,c_0) \right\|_{\infty} \le \epsilon .
\]
If $g(x) > 2\epsilon$, then the predicted class is invariant across all
allowed contexts.
\end{cprop}

\begin{proof}
Fix any allowed context $c$ and any competitor class $j \neq w$. The
winning logit can decrease by at most $\epsilon$, and the competitor logit
can increase by at most $\epsilon$. Therefore
\[
z_w(x,c) - z_j(x,c) \;\ge\; \bigl(z_w(x,c_0) - \epsilon\bigr) -
\bigl(z_j(x,c_0) + \epsilon\bigr) \;\ge\; g(x) - 2\epsilon \;>\; 0 .
\]
The reference winner remains strictly above every competitor in every
allowed context, so the predicted label does not change.
\end{proof}

The strict inequality is required for this norm bound alone to certify a
strict winner. Equality permits a tie but does not require one, so actual
invariance can still hold below the threshold. A sharper shift-invariant
condition replaces $2\epsilon$ by the oscillation of the perturbation,
$\operatorname{osc}(\Delta) = \max_k \Delta_k - \min_k \Delta_k \le
2\|\Delta\|_{\infty}$, which discounts any common scalar added to all
logits. This certificate connects to the empirical margin stratification,
in which precision-induced and execution-induced label changes concentrated
at small pre-intervention margins while wide-margin examples remained
stable.

\begin{cprop}{Sufficient probability-margin condition}{tv-margin}
Let $m(x) = p_w(x,c_0) - \max_{j \neq w} p_j(x,c_0)$ be the reference
probability margin, and for one example and context define
$\operatorname{TV}(x,c,c_0) = \tfrac{1}{2} \sum_y |p_y(x,c) - p_y(x,c_0)|$.
If
\[
m(x) > 2 \operatorname{TV}(x,c,c_0),
\]
then the predicted label cannot change in context $c$. For a guarantee over
all declared contexts, replace $\operatorname{TV}$ by its supremum over $c$.
\end{cprop}

\begin{proof}
For an event consisting of a single class, the absolute difference of its
probability between the two contexts is at most the total variation
distance. Hence the winner's probability can decrease by at most
$\operatorname{TV}(x,c,c_0)$, and any competitor's probability can increase
by at most $\operatorname{TV}(x,c,c_0)$. Every winner-versus-competitor
probability gap therefore decreases by at most
$2\operatorname{TV}(x,c,c_0)$, and since the reference gap $m(x)$ strictly
exceeds this amount, the winner remains strictly ahead of every competitor.
\end{proof}

\begin{cprop}{Corollary, necessary condition for a label change}{flip-necessary}
Under the assumptions of the two preceding propositions, if the predicted
label changes between contexts $c_0$ and $c$, then necessarily
\[
g(x) \le 2\epsilon
\qquad \text{and} \qquad
m(x) \le 2\operatorname{TV}(x,c,c_0).
\]
\end{cprop}

\begin{proof}
Both statements are the contrapositives of the sufficient conditions just
proved. If $g(x) > 2\epsilon$ the label cannot change, so a label change
forces $g(x) \le 2\epsilon$; if $m(x) > 2\operatorname{TV}(x,c,c_0)$ the
label cannot change, so a label change forces
$m(x) \le 2\operatorname{TV}(x,c,c_0)$.
\end{proof}

This necessary condition held for all 375 of 375 observed reduced-precision
label changes in the movie-review cell. Every change satisfied
$m \le 2\operatorname{TV}$. This is a consistency check implied by the
proposition, not evidence that the bound is a strong prospective predictor.
A prospective verifier would additionally need an intervention-specific
total variation upper bound available before accepting the fast result, and
the certificate covers only the top label, not calibration, ranking,
thresholding, filtering, or abstention.

\subsection{Quantization scale policy and peer coupling}

\begin{cprop}{Dynamic batch-global scale couples the target to its peers}{dynamic-scale}
For a positive scale $s$ and positive signed integer limit $M$, define the
dequantized scalar quantizer
\[
Q_s(a) = s \, \operatorname{clip}\!\bigl( \operatorname{round}(a/s),\, -M,\, M \bigr),
\]
and let a batch-global dynamic quantizer choose $s(B) = \max_{a \in B} |a| / M$
or an analogous runtime range statistic, where the batch $B$ contains the
target row and its peers (an all-zero batch requires an explicit positive
default scale). Let $B$ and $B'$ have the same shape and the same target
activation $a$ but different peer values. If
\[
s(B) \neq s(B')
\qquad \text{and} \qquad
Q_{s(B)}(a) \neq Q_{s(B')}(a),
\]
then the target's quantized activation depends on peer content. If in
addition the downstream target-logit map is injective on these two
activation points, or the relevant logit is strictly monotone along the
segment connecting them, then the target logits also differ.
\end{cprop}

\begin{proof}
The target value $a$ is fixed, but its quantized representation is
evaluated on two different quantization grids determined by the two scales.
The stated inequality $Q_{s(B)}(a) \neq Q_{s(B')}(a)$ makes the dequantized
target activation different in the two contexts. Injectivity of the
downstream map on these two points, or strict monotonicity of a relevant
logit along their connecting segment, then maps the two distinct
activations to two distinct logit values.
\end{proof}

This proposition establishes a possible scale-mediated peer dependence, not
inevitable movement. Some target values lie on both grids, downstream
computation can map distinct activations to the same logits, and the
formula models a batch-global quantizer rather than dynamic quantizers with
request-local, row-local, or token-local scales. It corresponds to the
dynamic int8 experiment, in which changing only peer content at fixed shape
moved target logits by up to 0.774 for the fully generative model and about
0.09 for the discriminative model, with up to 2 of 64 label changes, for
the one tested dynamic-quantization implementation and module set.

\begin{cprop}{Fixed peer-independent scale does not depend on peer content at fixed shape}{fixed-scale}
Assume all of the following:
\begin{enumerate}
\item The target tokens, masks, positions, initial activations, and
quantization state are fixed.
\item The activation scale is positive, fixed before serving, and
independent of peer values.
\item Every target-row intermediate depends only on the preceding
target-row intermediate and fixed constants.
\item Kernel selection, operation order, rounding path, and runtime state
used for the target are independent of peer values at the fixed shape.
\item Execution is deterministic under these conditions.
\end{enumerate}
Then changing peer content at fixed shape cannot change the target output.
\end{cprop}

\begin{proof}
By assumption 2 the quantizer scale does not respond to the peers, so the
fixed quantizer maps the fixed target activation to the same quantized
value in both contexts. By assumptions 3 and 4 the target-row computation
applies the same operations, in the same order, with the same rounding path
and runtime state, to the same target intermediates, regardless of peer
values. By assumption 5 there is no residual source of variation. Induction
over the target-row computation therefore yields identical target
intermediates at every stage and an identical target output.
\end{proof}

The proposition does not guarantee invariance across batch shapes, kernels,
hardware, precision, or scale-application order. The server static-int8 configuration
and the dynamic-scale control in the same evaluation setup both use
fake-quantized linear modules with fp32 accumulation and differ only in
the activation scale source. This comparison isolates that variable but
does not use a production integer kernel. The dynamic-int8 experiment in Table~\ref{tab:int8} uses the
integer CPU kernel, and the phone rows use exported integer graphs. The proposition predicts the observed
contrast between the dynamic and fixed-scale int8 experiments. The same
peer-content intervention that moved logits under a batch-dependent runtime
scale became bitwise identical on the target once the activation scale was
fixed offline and made independent of the peers.

\subsection{Fully generative scoring and greedy decoding}

\begin{cprop}{Length-proportional drift bound for summed sequence scores}{gen-drift}
For class $y$, let the token-level log-probability contribution under
context $c$ be $\ell_{y,t}(c) = \log p_{\theta,c}(x_t \mid x_{<t}, y)$ and
the class score be $S_y(c) = \sum_{t=1}^{L} \ell_{y,t}(c)$, with padding
positions excluded. Assume each token contribution changes by at most
$\epsilon$ across two contexts:
\[
|\ell_{y,t}(c) - \ell_{y,t}(c')| \le \epsilon \quad \text{for all } t .
\]
Then for each class,
\[
|S_y(c) - S_y(c')| \le L \epsilon ,
\]
and for any two classes $y$ and $j$, the score-gap drift is bounded by
\[
\bigl| [S_y(c) - S_j(c)] - [S_y(c') - S_j(c')] \bigr| \le 2 L \epsilon .
\]
\end{cprop}

\begin{proof}
For the first bound, apply the triangle inequality to the sum of $L$
token-level differences:
\[
|S_y(c) - S_y(c')| = \Bigl| \sum_{t=1}^{L} \bigl( \ell_{y,t}(c) -
\ell_{y,t}(c') \bigr) \Bigr| \le \sum_{t=1}^{L} |\ell_{y,t}(c) -
\ell_{y,t}(c')| \le L \epsilon .
\]
For the second bound, apply the triangle inequality once more to the
difference of the two class-score drifts:
\[
\bigl| [S_y(c) - S_j(c)] - [S_y(c') - S_j(c')] \bigr| \le
|S_y(c) - S_y(c')| + |S_j(c) - S_j(c')| \le 2 L \epsilon .
\]
\end{proof}

\begin{cprop}{Corollary, margin certificate for fully generative prediction}{gen-margin}
Under the assumptions of the preceding proposition, if the reference
class-score margin exceeds $2 L \epsilon$, then the fully generative
prediction is invariant under the assumed token-level perturbation bound.
If class-conditioned sequences have different lengths $L_y$ and $L_j$, the
pairwise bound is $(L_y + L_j)\epsilon$ rather than $2 L \epsilon$.
\end{cprop}

\begin{proof}
By the preceding proposition, every winner-versus-competitor score gap
changes by at most $2 L \epsilon$ between the two contexts (or by at most
$(L_y + L_j)\epsilon$ under class-dependent lengths). If the reference
margin strictly exceeds this amount, the reference winner remains strictly
ahead of every competitor, by the same argument as in the logit-margin
certificate.
\end{proof}

These bounds establish a larger admissible drift for summed sequence scores
than for a classifier that reads a single output position, whose corresponding factor is $2\epsilon$
rather than $2L\epsilon$. They do not prove larger realized drift, and
length normalization, shared computation, higher margins, contraction, or
exact arithmetic can reverse the ordering. The comparison isolates the
aggregation mechanism behind two empirical observations: the difference between formulations, in
which the fully generative model showed maximum batch-shape logit movement
of approximately 0.447 under reduced precision against 0.03125 for the
discriminative model in the tested cells, and the within-example trend of
drift increasing with real-token length at fixed padded width.

\begin{cprop}{First-divergence certificate for greedy decoding}{first-fork}
Consider two executions of greedy decoding under contexts $c$ and $c'$ from
the same prompt with a fixed deterministic tie rule. Before the two
trajectories diverge, their prefixes are identical. Let $g_t$ be the
top-two logit margin at step $t$ on the common prefix under the reference
context $c_0$, and suppose the
two-context logit difference at that same prefix is bounded in the infinity
norm by $\epsilon_t$. If
\[
g_t > 2 \epsilon_t \quad \text{at every step } t = 1, \dots, T,
\]
then the selected token is identical in both executions at every step
through $T$.
\end{cprop}

\begin{proof}
Proceed by induction on the decoding step. At step 1 the two executions
share the prompt as their common prefix. Assume the prefixes are identical
through step $t-1$. Then at step $t$ both executions evaluate the model on
the same prefix, and their logit vectors differ by at most $\epsilon_t$ in
the infinity norm. By the logit-margin certificate, the condition
$g_t > 2\epsilon_t$ guarantees that the same token is selected in both
executions, so the prefixes remain identical through step $t$. By induction
the selected tokens agree at every step through $T$.
\end{proof}

Once one token differs, the prefixes are no longer equal and the same-input
perturbation model no longer applies. Subsequent divergence can then
reflect ordinary sensitivity to different inputs rather than continued
numerical error accumulation alone. This certificate frames the
autoregressive accumulation experiment, in which reduced-precision batch-shape
perturbations changed all three tested generations within the first several
steps while full-precision generation did not change within the measured
horizon, and every same-context repeat was bitwise identical.

\subsection{Propagation of discrepancies through depth}

\begin{cprop}{Layerwise error propagation}{layerwise}
Let the layer state under context $c$ satisfy
$h_{\ell}^{c} = F_{\ell}^{c}(h_{\ell-1}^{c})$ for layers
$\ell = 1, \dots, L$. Assume that over all reachable states each
context-specific layer map is Lipschitz,
\[
\| F_{\ell}^{c}(u) - F_{\ell}^{c}(v) \| \le K_{\ell} \| u - v \| ,
\]
and that the same-state discrepancy between the two context-specific maps
is uniformly bounded,
\[
\| F_{\ell}^{c}(v) - F_{\ell}^{c'}(v) \| \le \epsilon_{\ell} .
\]
Then the final discrepancy satisfies
\[
\| h_L^c - h_L^{c'} \|
\le
\| h_0^c - h_0^{c'} \| \prod_{j=1}^{L} K_j
+
\sum_{\ell=1}^{L} \epsilon_{\ell} \prod_{j=\ell+1}^{L} K_j .
\]
\end{cprop}

\begin{proof}
For each layer $\ell$, insert and subtract the mixed term
$F_{\ell}^{c}(h_{\ell-1}^{c'})$:
\[
\| h_{\ell}^{c} - h_{\ell}^{c'} \|
\le
\| F_{\ell}^{c}(h_{\ell-1}^{c}) - F_{\ell}^{c}(h_{\ell-1}^{c'}) \|
+
\| F_{\ell}^{c}(h_{\ell-1}^{c'}) - F_{\ell}^{c'}(h_{\ell-1}^{c'}) \| .
\]
The first term is at most $K_{\ell} \| h_{\ell-1}^{c} - h_{\ell-1}^{c'} \|$
by the Lipschitz assumption, and the second is at most $\epsilon_{\ell}$ by
the same-state discrepancy assumption. This yields the recurrence
\[
\| h_{\ell}^{c} - h_{\ell}^{c'} \| \le K_{\ell}
\| h_{\ell-1}^{c} - h_{\ell-1}^{c'} \| + \epsilon_{\ell} .
\]
Iterating the recurrence from $\ell = 1$ to $L$, the initial discrepancy is
multiplied by every later Lipschitz factor, and each local map discrepancy
$\epsilon_{\ell}$ is multiplied only by the Lipschitz factors of the layers
after it, which gives the stated bound.
\end{proof}

The bound formally separates the two regimes measured in the paper. For a
numerical execution comparison with an identical represented input, the
initial term is zero and only accumulated local discrepancies remain,
whereas for structural padding that changes learned absolute position
embeddings the initial term is generally nonzero. No monotonic depth
conclusion follows without additional assumptions on the $K_{\ell}$ and
$\epsilon_{\ell}$ values, since contracting later layers can shrink earlier
discrepancies and perturbations can cancel. The observed rise of hidden
score drift with nominal model size is therefore a descriptive association
consistent with one possible parameter regime, not a consequence of depth
alone.

\subsection{Conditional infeasibility of accuracy, speed, and invariance}

\begin{cthm}{Conditional triple infeasibility}{triple}
Fix a finite label set, a task distribution over $(X, Y)$, a bounded set of
allowed serving contexts $\mathcal{C}$, a bounded workload, one hardware
and software stack, an implementation family $\mathcal{H}$, and predeclared
thresholds $\alpha$ (accuracy), $\tau$ (latency), and $\delta$ (probability
drift in total variation). Define worst-context accuracy
$A(f) = \inf_{c \in \mathcal{C}} \Pr[\arg\max_y p_f(y \mid X, c) = Y]$,
worst-case latency $T(f)$ over the bounded workload, and drift
$V(f) = \sup_{x, c, c'} \operatorname{TV}(x, c, c')$. Define the fast and
invariant subset
\[
\mathcal{H}_{FI} = \{ f \in \mathcal{H} : T(f) \le \tau, \; V(f) \le \delta \} .
\]
Assume there exists a measurable statistic $Q = q(X)$ such that, for every
$f \in \mathcal{H}_{FI}$ and every allowed context $c$, the final
prediction is a deterministic measurable function of $Q$:
\[
\widehat{Y}_f(c) = g_{f,c}(Q) .
\]
Define the Bayes accuracy available from $Q$ as
\[
A_Q^{\star} = \mathbb{E}_Q \Bigl[ \max_y \Pr(Y = y \mid Q) \Bigr] .
\]
If $A_Q^{\star} < \alpha$, then no implementation in $\mathcal{H}$ is
simultaneously $\tau$-fast, $\alpha$-accurate, and $\delta$-invariant.
\end{cthm}

\begin{proof}
Fix any $f \in \mathcal{H}_{FI}$ and any allowed context $c$. Because the
decision of $f$ under $c$ factors through $Q$,
\[
\Pr\bigl[ \widehat{Y}_f(c) = Y \bigr]
=
\mathbb{E}_Q \bigl[ \Pr( Y = g_{f,c}(Q) \mid Q ) \bigr]
\le
\mathbb{E}_Q \Bigl[ \max_y \Pr( Y = y \mid Q ) \Bigr]
=
A_Q^{\star}
<
\alpha .
\]
The bound holds for every allowed context, so it also bounds the infimum
over contexts that defines $A(f)$; hence $A(f) < \alpha$. Every fast and
invariant implementation therefore fails the accuracy threshold, and the
intersection of the three feasible sets is empty.
\end{proof}

The information-restriction premise, namely that every fast and invariant
implementation factors through the statistic $Q$, is an assumption of the
theorem and has not been established for the classifier family studied in
this paper. The result is a generic conditional feasibility theorem, not an
architecture-specific impossibility result. The theorem can also be vacuous
if the fast and invariant subset is empty or if $\alpha$ exceeds the
full-information Bayes accuracy, and a strict three-way incompatibility
would additionally require three separately established pairwise witnesses
under common thresholds. The theorem's role is interpretive. The measured
mitigation costs (batch-invariant execution at an end-to-end latency ratio
near 1.3 to 1.5 at large batch and roughly neutral at batch size 1, despite
isolated invariant matrix multiplications running several times slower per
call, and position assignment that does not depend on the padded width
preserving family-level quality) show that the fast and invariant subset
need not be empty in practice, which constrains how the theorem can be
instantiated.

\begin{cprop}{Approximate bottleneck bound}{approx-bottleneck}
Suppose that for every fast and invariant implementation
$f \in \mathcal{H}_{FI}$ there exists a family of deterministic $Q$-based
rules indexed by context, $\{ h_{f,c} : c \in \mathcal{C} \}$, such that
\[
\sup_{c \in \mathcal{C}} \Pr\bigl[ \widehat{Y}_f(c) \neq h_{f,c}(Q) \bigr] \le \rho .
\]
Then every such implementation satisfies
\[
A(f) \le A_Q^{\star} + \rho ,
\]
and if $A_Q^{\star} + \rho < \alpha$, the triple of thresholds remains
infeasible.
\end{cprop}

\begin{proof}
Fix any context $c$. The event that $f$ predicts correctly under $c$ is
contained in the union of the event that $h_{f,c}(Q)$ predicts correctly
and the event that the two rules disagree. Therefore
\[
\Pr\bigl( \widehat{Y}_f(c) = Y \bigr)
\le
\Pr\bigl( h_{f,c}(Q) = Y \bigr)
+
\Pr\bigl( \widehat{Y}_f(c) \neq h_{f,c}(Q) \bigr) .
\]
The first term is at most $A_Q^{\star}$, since $h_{f,c}$ is a deterministic
function of $Q$ and no such rule can exceed the Bayes accuracy available
from $Q$; the second term is at most $\rho$, uniformly over contexts.
Taking the infimum over contexts gives $A(f) \le A_Q^{\star} + \rho$. If
$A_Q^{\star} + \rho < \alpha$, then $A(f) < \alpha$ for every fast and
invariant implementation, so no implementation satisfies all three
thresholds.
\end{proof}

This relaxation weakens exact factorization to approximate agreement with a
$Q$-based rule. Both results take the restriction to $Q$ as a premise. The
remainder of this subsection derives that premise from an execution model,
exhibits a finite instance in which the three thresholds are pairwise
feasible but jointly infeasible, and extends the derivation to the
approximate setting of Proposition~\ref{propx:approx-bottleneck}.

\begin{cprop}{Collision lemma, exact invariance factors through the fast representation}{collision}
Let the input support $\mathcal{X}$ be finite. Suppose that for every
context $c \in \mathcal{C}$ the fast execution path exposes a representation
$R_c(x)$ of the input, and that every implementation with $T(f) \le \tau$
predicts from it, $p_f(\cdot \mid x, c) = h_{f,c}(R_c(x))$, where the
postprocessor $h_{f,c}$ may depend on the context but has no other access to
$x$. Form the undirected graph on $\mathcal{X}$ that joins $x$ and $x'$
whenever $R_c(x) = R_c(x')$ for some $c \in \mathcal{C}$, and let $Q(x)$
denote the connected component of $x$. If $f$ is $\tau$-fast and
label-invariant across $\mathcal{C}$, then its predicted label is a function
of $Q$; if $f$ is exactly score-invariant ($V(f) = 0$), then its whole
probability vector is a function of $Q$.
\end{cprop}

\begin{proof}
Let $x$ and $x'$ be joined by an edge witnessed by $c$, so
$R_c(x) = R_c(x')$. Then $p_f(\cdot \mid x, c) = h_{f,c}(R_c(x)) =
h_{f,c}(R_c(x')) = p_f(\cdot \mid x', c)$, so in context $c$ the fast path
cannot distinguish the two inputs. Label invariance gives
$\widehat{Y}_f(x, c') = \widehat{Y}_f(x, c) = \widehat{Y}_f(x', c) =
\widehat{Y}_f(x', c')$ for every $c'$, so the predicted label agrees on the
two endpoints in every context. Agreement is transitive along paths, hence
the label is constant on each connected component and is a function of
$Q$. Under $V(f) = 0$ the same chain of equalities holds for the full
probability vector.
\end{proof}

The lemma supplies the premise of Theorem~\ref{thmx:triple} with
$Q$ equal to the component map, so that
$A(f) \le A_Q^{\star}$ for every fast and exactly invariant implementation.
The collision must occur within a single context. $R_c(x) = R_{c'}(x')$
for $c \ne c'$ carries no constraint, because a context-aware
postprocessor may treat the two contexts differently. The lemma also
locates what an instantiation must establish. It must identify a lossy
$R_c$ for the actual fast path and show that every implementation inside
the latency budget is confined to it. Measuring the latency of particular
kernels does not by itself establish that lower bound over
implementations. If distinct fast policies expose distinct maps, each
policy has its own graph and its own bound.

\begin{cprop}{A finite instance in which each pair of thresholds is feasible and the triple is not}{four-input}
Let $X$ be uniform on $\{a, b, c, d\}$ with deterministic labels
$0, 1, 0, 1$, and let $\mathcal{C} = \{c_1, c_2, c_3\}$ with fast
representations given by the partitions
$c_1: \{a, b\}, \{c\}, \{d\}$; \;
$c_2: \{a\}, \{b, c\}, \{d\}$; \;
$c_3: \{a\}, \{b\}, \{c, d\}$,
where inputs in one cell share the same $R_c$. Let the fast path cost one
unit and a slow path that exposes the input identity cost two units, and
set $\tau = 1$, $\delta = 0$, and any $\alpha \in (1/2, 3/4]$. Then each pair
of thresholds is attainable while the triple is not.
\end{cprop}

\begin{proof}
In the fast and accurate case, predict in each context the true label on the two
singleton cells and one fixed label on the merged pair; exactly one of the
four inputs is wrong, so the accuracy is $3/4 \ge \alpha$ in every context,
but the predictions differ across contexts. In the accurate and invariant case, the slow
path returns the true label everywhere, with accuracy $1$, at cost $2 > \tau$.
In the fast and invariant case, a constant label has accuracy $1/2$. For the triple,
the collision graph has edges $a$--$b$ (witnessed by $c_1$), $b$--$c$
($c_2$), and $c$--$d$ ($c_3$), so it is connected and $Q$ is a single
component. By Proposition~\ref{propx:collision} every fast and invariant
rule is constant, and $A_Q^{\star} = 1/2 < \alpha$.
\end{proof}

No single context forces the accuracy down to $1/2$. The loss arises from
requiring agreement across contexts whose collisions overlap. The costs
define the abstract execution model of the example and are not measured
latencies.

\begin{cprop}{Approximate collision bound}{approx-collision}
In the setting of Proposition~\ref{propx:collision}, suppose every
connected component has graph diameter at most $D$, and let
$f$ be $\tau$-fast with $V(f) \le \delta$. Fix a reference context $c_0$ and
one representative $r(q)$ in each component $q$. Then for every $x$,
\[
\operatorname{TV}\bigl( p_f(\cdot \mid x, c_0),\; p_f(\cdot \mid r(Q(x)), c_0) \bigr)
\le 2 D \delta ,
\]
and, writing $m_f(x, c_0)$ for the top-two probability margin of
$p_f(\cdot \mid x, c_0)$ and $\rho_f = \Pr\{ m_f(X, c_0) \le 4 D \delta \}$,
\[
A(f) \le A_Q^{\star} + \rho_f .
\]
\end{cprop}

\begin{proof}
For an edge $x$--$x'$ witnessed by $c$, the fast path gives
$p_f(\cdot \mid x, c) = p_f(\cdot \mid x', c)$, and $V(f) \le \delta$ bounds
the movement of each endpoint between $c$ and $c_0$ by $\delta$; the
triangle inequality gives
$\operatorname{TV}(p_f(\cdot \mid x, c_0), p_f(\cdot \mid x', c_0)) \le 2\delta$.
Composing along a path of length at most $D$ to $r(Q(x))$ gives the first
bound. Proposition~\ref{propx:tv-margin}, applied to the two vectors
$p_f(\cdot \mid x, c_0)$ and $p_f(\cdot \mid r(Q(x)), c_0)$, states that a
total-variation distance of at most $2 D \delta$ cannot change the predicted
label when the margin exceeds $4 D \delta$; on the complement of the event
$\{ m_f(X, c_0) \le 4 D \delta \}$, $f$ therefore predicts at $c_0$ the
label of the $Q$-based rule $x \mapsto \widehat{Y}_f(r(Q(x)), c_0)$.
Proposition~\ref{propx:approx-bottleneck} with this rule and $\rho = \rho_f$
gives $\Pr[\widehat{Y}_f(c_0) = Y] \le A_Q^{\star} + \rho_f$, and
$A(f)$ is at most the accuracy at any single context.
\end{proof}

The approximate bound derives the disagreement rate $\rho$ of
Proposition~\ref{propx:approx-bottleneck} from four quantities: the
collision structure, the component diameter, the score tolerance, and the
margin tail at the reference context. A family-wide statement needs a
uniform bound $\rho_f \le \rho$ over the fast and $\delta$-invariant
implementations and $\alpha > A_Q^{\star} + \rho$. The bound can be loose.
If $D$ is large, or $A_Q^{\star} + \rho$ reaches one, it excludes nothing.
These results establish the theorem's premise for any execution model of
the stated form. They do not identify such a model for the GPU and CPU
paths measured in this paper, and the measured mitigation costs indicate
that speed, quality, and invariance were not observed to be pairwise
incompatible in the tested configurations. No controlled joint frontier
over all three objectives has been measured.

\clearpage
\section{Extended Limitations}
\label{app:limitations}

The study isolates each serving factor on a fixed checkpoint and fixed text,
and the boundaries below state where its measured magnitudes apply. The
mechanism results are stated in terms of the position rule and the top-two
margin, which do not depend on these boundaries.

\paragraph{Model and task scope.}
We study transformer text classifiers, trained from initialization in the
primary grid following the protocol of \citet{kasa2025generative} and
fine-tuned in the pretrained controls, in four formulations: a
discriminative encoder head, a masked-language-model label predictor,
a GPT-2 classifier that predicts the next label token, and a GPT-2 classifier that scores the input under each class. We focus on conventional
fine-tuning and include LoRA as the parameter-efficient control on the two
modern backbones; in-context classification with instruction-tuned language
models and other adaptation methods are left for future work, where the
same position and margin analysis applies. The inputs are unimodal text, so
classifiers that combine text with tabular, image, or audio features remain
to be measured.

\paragraph{Hardware and backend scope.}
Measurements run on one representative GPU (an NVIDIA L40S), one
representative CPU (an AMD EPYC), and two Android phones, with PyTorch eager
execution compared against torch.compile and an ONNX Runtime export. The
reported magnitudes are specific to these configurations. Production engines such as
vLLM and SGLang, tensor-parallel execution, and other accelerators change
the arithmetic through their own kernel choices and require their own
measurement, for which the repeatability floor and target-anchored design
transfer unchanged. The GPU-fp32 to CPU-int8 comparison is reported as a
whole migration; Appendix~\ref{sec:app-device-quant} separates its device
and quantization components on matched subsets.

\paragraph{Evaluation scope.}
We evaluate on standard independent and identically distributed test
splits. Behavior under distribution shift, where both the baseline
predictions and the perturbation sensitivity may differ, is a natural
extension. The balanced grids report how often an effect occurs under a
designed set of conditions rather than its incidence in production traffic,
and each serving context is executed once, with same-context run-to-run
variation measured separately by the repeatability floor.

\paragraph{Mechanism and architecture scope.}
The architecture comparison holds the reference-code conventions of each
family fixed rather than randomizing padding side, position representation,
pretraining, and formulation independently. The Qwen3 rotary control and the
direct position-index intervention measure the two cases that the position
theorem separates, and the statements about ALiBi, T5, Shaw, Transformer-XL,
and sinusoidal encodings follow from their published formulas. The
from-scratch size series scales depth, width, and head count jointly and
conditions on all three sizes passing the accuracy threshold, so its trends
describe how the effect moves with capacity at matched data and seeds. The
no-change score summaries are computed per size, and probability total
variation provides the scale-free comparison across independently trained
checkpoints alongside raw logit movement.

\paragraph{Pretrained-model scope.}
The pretrained GPT-2 experiment establishes padding vulnerability on public
checkpoints. Because changing the position rule only at inference gives a
model position inputs it was never trained on, the accuracy threshold excludes
that configuration (it collapses five of fifteen models to a constant
prediction: SST-2 seeds 41 and 42 and SST-5 seeds 40, 41, and 42), and the
mitigation evidence comes from the three-seed experiment that trains and
serves the same rule, which establishes agreement on the predicted label.
The pretrained replication covers one 124M-parameter GPT-2 family and
formulation, target-first batch construction, and five dataset-level
bootstrap clusters; Appendix~\ref{sec:app-mobile} and the modern-backbone
controls extend the family coverage, and larger pretrained families are the
next measurement.

\end{document}